\documentclass[11pt]{article}

\usepackage[utf8]{inputenc}
\ifdefined\pdfgentounicode
  \input{glyphtounicode}
\fi
\usepackage{amsmath,amssymb,amsthm}
\usepackage{mathtools}
\usepackage{bm}
\usepackage{geometry}
\usepackage{enumitem}
\usepackage{booktabs}
\usepackage{hyperref}
\hypersetup{colorlinks=true,linkcolor=blue,citecolor=blue,urlcolor=blue}
\usepackage{cleveref}
\usepackage{graphicx}
\usepackage{tikz}
\usetikzlibrary{arrows.meta,positioning,fit,backgrounds}
\usepackage[numbers,sort&compress]{natbib}
\usepackage{xcolor}
\usepackage{listings}
\theoremstyle{definition}
\newtheorem{definition}{Definition}[section]
\newtheorem{assumption}{Assumption}[section]
\newtheorem{proposition}{Proposition}[section]
\newtheorem{lemma}{Lemma}[section]
\newtheorem{corollary}{Corollary}[section]
\newtheorem{remark}{Remark}[section]
\newtheorem{example}{Example}[section]
\newtheorem{prediction}{Prediction}[section]
\theoremstyle{plain}

\newcommand{\R}{\mathbb{R}}
\newcommand{\Orth}{\mathrm{O}}
\newcommand{\norm}[1]{\left\lVert #1 \right\rVert}
\newcommand{\inner}[2]{\left\langle #1,\, #2 \right\rangle}
\newcommand{\HH}{\bm{H}}
\newcommand{\GG}{\bm{K}}
\newcommand{\QQ}{\bm{Q}}
\newcommand{\MM}{\bm{M}}
\newcommand{\Grp}{\mathcal{G}}
\newcommand{\Probe}{\mathcal{X}}
\newcommand{\Win}{\bm{W}_{\mathrm{in}}}
\newcommand{\Labs}{L_{\mathrm{abs}}}
\newcommand{\Obj}{\mathcal{L}}
\newcommand{\WW}{\bm{W}}
\newcommand{\Lam}{\bm{\Lambda}}
\newcommand{\rank}{\operatorname{rank}}
\newcommand{\tr}{\operatorname{tr}}
\newcommand{\CKA}{\mathrm{CKA}}
\DeclareMathOperator*{\argmin}{arg\,min}

\title{\textbf{Teacher Supervision over Representation Equivalence Classes:
Why Distillation Transfers Capability Through the Output Function, Not the Features}}

\author{}
\date{}

\begin{document}

\begin{center}
  {\LARGE\bfseries
   Teacher Supervision over Representation\\[2pt]
   Equivalence Classes:\\[4pt]
   {\large Why Distillation Transfers Capability Through the Output Function,
   Not the Features}}\\[12pt]
  {\large Sang-Il Han$^{1,2}$}\\[6pt]
  {\normalsize
   $^1$Korea University of Technology and Education
   \quad
   $^2$Biomonster Co.}\\[2pt]
  {\normalsize
   \href{mailto:sihan@koreatech.ac.kr}{\texttt{sihan@koreatech.ac.kr}}}\\[4pt]
  {\small April 2026}\\[8pt]
\end{center}

\begin{abstract}
Knowledge distillation is usually framed as a choice of \emph{what} to match in
the teacher---its logits, hidden features, or sample relations---which
presupposes that the teacher's representation has absolute coordinates to match.
It does not: a pretrained representation is identifiable only up to an
orthogonal--and--isotropic--scaling equivalence class, so a student should learn
the teacher's \emph{equivalence class}, not its features. The organizing fact is
that capability is the teacher's \emph{output function}, a class invariant that
factors through the quotient by the class action---so an objective recovers
capability exactly when it is defined there. This makes absolute feature matching
ill--posed and admissible supervision a matter of targeting class invariants
(Gram structure, CKA, principal subspaces) or aligning coordinates first,
unifying feature matching, relational distillation, alignment, and grafting in
one geometric account. We validate on Qwen2.5 and Llama--3.1: a restoration study
recovers a corrupted model's representation ($\CKA\approx0.99$) but \emph{not} its
capability, and an ablation isolates the cause---output--function (logit) matching
drives capability while matching hidden representations aligns geometry without
restoring function. Recovery is confined to the corpus--covered region, and a graft
study confirms boundary overlap predicts transplant success but is necessary,
not sufficient.
\end{abstract}

\section{Introduction}
\label{sec:intro}

Here is a puzzle. Take a corrupted language model and restore its internal
representation until it almost exactly matches the teacher's---per--layer
$\CKA\approx0.99$, by a label--free signal. The model stays broken: its
perplexity is $10^2$\,--\,$10^4\times$ worse and its top--1 agreement with the
teacher is near zero. Now instead match the teacher's \emph{output
distribution} and ignore the internal representation: capability returns
(perplexity ratio $\approx1.01$, top--1 $0.98$), even though the representation
is no better aligned than before. Matching what is inside the network does not
transfer what the network can do; matching what comes out of it does. Why?
This paper's answer is geometric, and it is the reason for the machinery that
follows: a representation has no absolute coordinates, so ``match the features''
is the wrong target---capability is an \emph{invariant} of the representation's
equivalence class, namely its output function, and only an objective defined on
that invariant transfers it.

Before any of that machinery, it helps to fix what a model \emph{is}. A network
maps an input $x$ to an output $f(x)=W\!\bigl(H(x)\bigr)$: the input is carried to
an internal representation $H(x)$, and a reader $W$ (the downstream weights) turns
that representation into the output. Because the reader can absorb any rotation
or rescaling of the representation, the same function $f$ is computed by a whole
family of pairs $(H,W)$, so a model is properly not a single $(H,W)$ but the
\emph{class} of all pairs computing the same output. Capability is a property of
that class, not of $H$ alone---which is exactly why matching $H$ can leave the
function untouched. The underlying fact fits in a few lines of code.
A representation is just an intermediate activation, and the layer that reads it
can absorb a change of its coordinate frame:
\begin{center}
\begin{minipage}{0.74\linewidth}
\begin{verbatim}
h  = x @ W1               # a hidden representation
y  = h @ W2               # the layer that reads it -> output

Q  = random_rotation()    # any orthogonal change of basis
h2 = h @ Q                # rewrite the representation...
W2 = Q.T @ W2             # ...and the reader, together

assert allclose(y, h2 @ W2)   # same output, always
\end{verbatim}
\end{minipage}
\end{center}
The rotated network stores completely different numbers in \texttt{h2} yet
computes the identical function. So the specific coordinates of \texttt{h} carry
no information the model can use: they are an arbitrary frame, fixed only up to
the rotations (and rescalings) the next layer undoes. This is the whole premise
in miniature---\emph{a feature has coordinates, the coordinates have no meaning,
and only the function is determined}---and it is exactly why matching a teacher's
features basis--for--basis differs from matching its output, which is invariant
to the rewrite above. The rest of the paper turns
this three--line \texttt{assert} into a statement about equivalence classes and
the quotients they form; \cref{sec:code-appendix} gives the same three objects
this induces---orbit, fiber, quotient---as runnable code for the reader who
prefers to meet them that way first.

Two words in that sentence carry the whole paper, and both have exact analogues
a systems or compiler reader already owns. A \emph{feature} (the vector
\texttt{h} above) is an \emph{internal representation used by a computation}---the
neural--network analogue of a compiler's intermediate representation (IR). Just
as renaming an SSA temporary changes the IR while computing the same program,
$\texttt{h}\mapsto\texttt{h @ Q}$ changes the feature while computing the same
function: the coordinates of an IR, or of a feature, are an implementation
detail, not meaning. An \emph{invariant} is the dual notion---the quantity that
does \emph{not} change under such a rewrite. This paper's one empirical claim, in
that vocabulary, is simple: the invariant a downstream consumer can rely on is
the \emph{output function}, not the feature. The feature is the model's IR, free
to be recoordinatized; capability is the invariant.

Distillation and module transfer both raise the same question: in what
coordinate system does the ``knowledge'' of a representation live? The
established lines of distillation each answer a \emph{what--to--match} question:
logit distillation~\citep{hinton2015} matches output distributions, feature
distillation~\citep{romero2015} matches hidden activations, and relational
distillation~\citep{park2019,tung2019} matches sample--to--sample structure.
All three presuppose that there is a well--defined teacher quantity to match.
We step back one level and question that premise. A student should reproduce
the teacher's features basis--for--basis only if those features have an
absolute basis to begin with---and they do not: representations are
unidentifiable up to a group of transformations. The consequence is not a new
loss but a change in what teacher supervision \emph{is}. The teacher is not a
set of feature vectors but an equivalence class of them, and distillation is
the transfer of that class (made precise in \cref{def:distill}). We make this precise: absolute--coordinate matching
is ill--posed, and transfer must
target the \emph{invariants} of the representation's equivalence class, or
operate in coordinates fixed by an explicit alignment. Read this way, feature
matching, relational distillation, alignment--based transfer, and module
grafting are not separate techniques but different points in one geometric
formulation, distinguished by which part of the class structure they fix.

The distinction from prior work is sharp, and worth stating before a reader
files this under ``another use of CKA.'' The difference is one of \emph{question},
not method. Representation--similarity research
\citep{kornblith2019,klabunde2023} asks \emph{how alike} two representations
are, and answers with a diagnostic number. We ask a different question:
\emph{what transfers a capability?} We therefore propose no new similarity
measure and no new loss. We change the \emph{object} of distillation: the
representation is itself an equivalence class, and ``what should we match''
becomes ``what is defined on the quotient the output function factors through.''

This reframing places $\CKA$ precisely. It is one coordinate on the
representation quotient---the \emph{coarser} of the two, which forgets the reader
(\cref{sec:invariants}). Capability instead lives on the finer, joint quotient
(\cref{prop:quotient}), so $\CKA$ is the \emph{wrong} coordinate for transferring
it. A similarity paper would stop at measuring it; our point is the opposite. A
high $\CKA$ can coexist with zero capability transfer (\cref{tab:ablation}),
precisely because the two live on different quotients.

Our contribution is a single conceptual reframing of teacher supervision---from
matching teacher \emph{features} to transferring the teacher's representation
\emph{equivalence class}---in three parts.
\textbf{Theory:} absolute--coordinate feature matching is ill--posed
(\cref{sec:unident}), and the output function is the class invariant capability
depends on, factoring through the quotient as $f=\bar f\circ\pi_{\mathrm{joint}}$
(\cref{prop:quotient}).
\textbf{Loss taxonomy:} this factorization sorts basis--invariant objectives and
locates standard logit and feature distillation within it (\cref{sec:objectives}),
and predicts graft success from subspace overlap (\cref{sec:graft}).
\textbf{Experiment:} restoration, ablation, and graft studies on Qwen2.5 and
Llama--3.1 validate the account (\cref{sec:experiments,sec:graft}).

These are not three separate stories but one: teacher supervision transfers an
equivalence class, and its operational core is a single distinction---a
representation is an equivalence class, but capability is the output function, a
class invariant (\cref{prop:quotient}).

\paragraph{Terms, fixed once.} We use four words in fixed senses. The
\emph{representation} is the layer activations $\HH$, defined only up to the class
$[\HH]$ (\cref{def:representation}); the \emph{output function} $f$ maps input to
logits and is a joint--class invariant (\cref{prop:quotient}); \emph{capability}
is any scalar read off that function (perplexity, top--1)---deliberately narrow,
\emph{not} a claim about reasoning or generalization beyond what the logits
encode; and a teacher's \emph{knowledge}, informally, means exactly that output
function, the transferable content. The thesis relates them: capability rides on
the output function, invariant to the representation's coordinates. The narrowness
has teeth---one--step functionals can be restored while trajectory behavior is
not (\cref{res:crosswidth})---so ``capability transfers'' always means this
output--function sense and no more.

Each experiment tests a consequence of this one
statement. The restoration study (\cref{sec:experiments}) shows the two come
apart---aligning the representation ($\CKA\!\approx\!1$) does not restore the
function; restoring the function (output--function matching) does not require
aligning the representation. The decoder--from--scratch result shows what
follows when only the class--invariant output function is targeted: the
teacher is rebuilt on whatever subspace the corpus covers. The graft study
shows the same geometry governs a different operation---module transfer
succeeds to the extent the boundary subspaces overlap. Read together, they say
that the teacher's transferable content is its output function (a class
invariant), and ``what transfers'' is the part of that function one's objective
and corpus actually target.

\begin{figure}[h]
\centering
\begin{tikzpicture}[
  >={Stealth[length=2.5mm]},
  node distance=20mm and 13mm,
  stage/.style={draw, rounded corners, align=center, inner sep=3.5pt,
                font=\footnotesize, minimum height=7mm, text width=22mm},
  theory/.style={stage, fill=black!4},
  pred/.style={stage, fill=black!4},
  expt/.style={stage, fill=black!4},
  every edge/.style={draw, ->, thick}
]
\node[theory] (a) {\textbf{Assumption 2.1}\\ rotations are absorbable};
\node[theory, right=of a] (p2) {\textbf{Prop.\ 2.1}\\ absolute matching ill--posed};
\node[theory, right=of p2] (p32) {\textbf{Prop.\ 3.2}\\ output function $=$ class invariant; capability follows $f=\bar f\!\circ\!\pi_{\mathrm{joint}}$};
\node[expt, below=11mm of p32] (r2) {\textbf{Result 2}\\ ablation: $L_{\mathrm{logit}}$ recovers capability, $L_{\CKA}$ does not};
\node[pred, left=of r2] (pr) {\textbf{Pred.\ 6.1--6.3}\\ graft success $\sim$ overlap};
\node[expt, left=of pr] (disc) {\textbf{Discussion}\\ replication, not creation};

\draw[->, thick] (a) -- (p2);
\draw[->, thick] (p2) -- (p32);
\draw[->, thick] (p32) -- (r2);
\draw[->, thick] (r2) -- node[font=\scriptsize, above]{validates} (pr);
\draw[->, thick] (pr) -- (disc);
\draw[->, thick] (p2.south) to[bend right=12] (pr.north);
\end{tikzpicture}
\caption{The logical spine. \cref{ass:absorb} (rotations can be absorbed) makes
absolute feature matching ill--posed (\cref{prop:illposed}) and identifies the
output function as the class invariant capability depends on
(\cref{prop:quotient}, the central result). The ablation of Result~2
(\cref{res:ablation}) is its experimental counterpart; the same geometry yields
falsifiable graft predictions (Predictions~\ref{pred:mono}--\ref{pred:align}),
validated in \cref{sec:graft}; the Discussion places the whole as
\emph{replication}, not creation. Every later section is a consequence of the
single premise on the left.}
\label{fig:spine}
\end{figure}

\section{Representational Unidentifiability}
\label{sec:unident}

Fix a finite \emph{probe set} $\Probe=\{x_1,\dots,x_N\}$ of inputs, and let
$\HH=\HH(\Probe)\in\R^{N\times d}$ collect the $d$--dimensional representations
of a chosen layer, one row $h_\ell(x_i)$ per probe. All representational
quantities below---$\HH$, its Gram matrix, $\CKA$, the equivalence class
$[\HH]$, and the quotients of \cref{sec:invariants}---are defined relative to
$\Probe$: they describe what the representation encodes \emph{about this probe
set}. We keep $\Probe$ fixed throughout a given comparison and suppress it in
the notation when no ambiguity arises, writing $\HH$ for $\HH(\Probe)$; where the
choice of probes matters (the corpus--coverage results of
\cref{sec:experiments}) we make the dependence explicit. The pointwise object
$h_\ell(x)\in\R^d$ of \cref{prop:outinv} is the single--input representation; the
probe matrix $\HH(\Probe)$ stacks it over $\Probe$. Throughout, $\Orth(d)=\{\QQ\in
\R^{d\times d}:\QQ^\top\QQ=\QQ\QQ^\top=\bm{I}\}$ denotes the orthogonal group in
dimension $d$ (rotations and reflections; $\QQ^{-1}=\QQ^\top$), acting on
representations by right multiplication $h\mapsto h\QQ$.

\begin{definition}[Representation]
\label{def:representation}
Fix a probe set $\Probe=\{x_1,\dots,x_N\}$ and a layer $\ell$ of the network. The
\emph{representation} at that layer is the matrix
\begin{equation}
\label{eq:representation}
\HH=\HH(\Probe)\in\R^{N\times d},
\qquad \text{row } i \text{ of } \HH \;=\; h_\ell(x_i)\in\R^{d},
\end{equation}
whose $i$th row is the layer--$\ell$ hidden activation elicited by probe $x_i$
(the residual--stream state after block $\ell$, in a transformer). It is a
purely empirical object: the network's internal features \emph{sampled} on
$\Probe$, and it is defined only relative to $\Probe$---a different probe set
yields a different $\HH$. The group $\Grp=\Orth(d)\times\R_+$ acts on it by
$\HH\mapsto c\,\HH\QQ$; this action is what \cref{def:equiv} quotients to form
the equivalence class $[\HH]$.
\end{definition}

\emph{Naming, fixed once.} Throughout, $\HH$ (the layer activations) is the
\emph{representation} and $\WW$ (specifically $\Win^{(\ell)}$) is the
\emph{reader}; the function depends on the pair together, so neither alone is
``the representation'' (\cref{app:naming} defends this usage against the
representation--theory and computer--science readers' objections).

A released model fixes $\theta$, not a representation $\HH$: the shipped weights
commit to one coordinate frame of the joint class, which is why
coordinate--sensitive operations across independently released models need an
explicit alignment (\cref{def:align}). \Cref{app:released} makes this precise.

Intuitively, let $\Win^{(\ell)}$ denote the first downstream operator that
consumes the representation $h_\ell$, that is, the linear map carrying one layer
to the next,
\[
h_\ell \ \overset{\Win^{(\ell)}}{\longrightarrow}\ h_{\ell+1}.
\]
Assumption 2.1 states that any admissible coordinate change of $h_\ell$ can be
absorbed by a counter--reparametrization of $\Win^{(\ell)}$, leaving the network
function unchanged---rotate the features one way, rotate the weights that read
them the other way, and the two cancel. (We write $\Win$ for $\Win^{(\ell)}$
when the layer is fixed.) This is the single premise from which the rest of the
paper follows. \Cref{app:status} discusses when it holds exactly (a plain MLP)
and when only approximately (transformers, where normalization can break it).

\begin{assumption}[Absorbable reparametrization]
\label{ass:absorb}
Let the map from the layer--$\ell$ representation to the output factor as
$\phi_{\ell{:}L}=\rho\circ\psi:\R^{d}\to\R^{V}$, where the first stage consumes
the representation through a linear operator $\Win\in\R^{d\times m}$ followed by a
nonlinearity, $\psi(h)=\tilde\psi(h\,\Win)$, and any normalization between $h$
and $\Win$ commutes with $\QQ$ (exactly on the subgroup of \cref{app:status}).
Then for $\QQ\in\Orth(d)$ and $c\in\R_+$ the paired actions $h\mapsto c\,h\QQ$ and
$\Win\mapsto c^{-1}\QQ^{\top}\Win$ leave the consumed product invariant, giving
the \emph{absorbing identity}
\begin{equation}
\label{eq:absorb}
\underbrace{\tilde\psi\bigl((h\QQ)(\QQ^{\top}\Win)\bigr)=\tilde\psi(h\Win)}
_{\text{(\ref{eq:absorb}a) cancellation: }\QQ\QQ^\top=\bm{I}}
\quad\Longrightarrow\quad
\underbrace{\phi_{\ell{:}L}^{\QQ}(h\QQ)=\phi_{\ell{:}L}(h)}
_{\text{(\ref{eq:absorb}b) function unchanged}},
\end{equation}
so the full group acting absorbably is $\Grp=\Orth(d)\times\R_+$.
\end{assumption}

A caveat on the group. The full $\Orth(d)$ action is exact in the idealized
linear/MLP case; in a real transformer, where attention $Q/K/V$ projections, MLP
projections, the residual stream, and RMSNorm gains are coupled, the exact group
is the \emph{architecture--preserving subgroup} of reparametrizations that every
downstream operator can absorb, with the full $\Orth(d)$ serving as the
organizing approximation. Everything below is stated for $\Orth(d)$ for clarity;
the claims that matter---ill--posedness of absolute matching, invariance of the
output function---hold verbatim on the subgroup, since it is a subgroup, and the
appendix MLP (\cref{tab:mlp}) exhibits the exact case.

Here $\phi_{\ell{:}L}^{\QQ}$ denotes $\phi_{\ell{:}L}$ with $\Win$ so
reparametrized; the cancellation is $(c\,h\QQ)(c^{-1}\QQ^{\top}\Win)=h\Win$, the
two actions exactly inverse on the product $h\Win$ that every downstream
computation sees. We state \cref{eq:absorb} for $c=1$ to keep the rotation
visible; the scale factor reappears in \cref{prop:outinv,prop:quotient}.
Conventions: $h\in\R^{1\times d}$ is one row $h_\ell(x)$ of the probe matrix, so
the assumption is pointwise in $x$ and the statement for $\HH(\Probe)$ follows by
stacking over rows, using row vectors throughout.

\Cref{app:transformer} instantiates $\Win$, $\psi$, and $\rho$ concretely for a
transformer (the residual stream and its query/key/value projections) and for a
plain MLP, where the absorbing identity is exact.

\emph{Exact vs.\ approximate.} The absorbing identity is exact when the
normalization between $\HH$ and $\Win$ commutes with $\QQ$---plain RMSNorm does;
a learned per--channel gain $\bm{\gamma}$ does so only on the subgroup it
preserves, and is otherwise approximate with error set by the spread of
$\bm{\gamma}$ (\cref{app:status}).

\emph{Status of the propositions.} On the subgroup $\Grp_0$ where
\cref{ass:absorb} holds exactly, \cref{prop:illposed,prop:outinv,prop:quotient}
are \emph{exact theorems}; only the \emph{size} of the equivalence class is
approximate for general $\QQ$, so the qualitative claims hold as soon as $\Grp_0$
is nontrivial, which it always is (\cref{app:status}).

The remainder of the paper develops the consequences of \cref{ass:absorb}, one
per section. Once some non--identity $\QQ$ preserves the function, the
representation is identified only up to $\QQ$: it is an equivalence class, not a
fixed matrix (\cref{def:equiv}). Matching absolute coordinates is then ill--posed
(\cref{prop:illposed}), so only class invariants are admissible targets
(\cref{sec:invariants}). The output function is invariant under the joint action,
hence a function on the quotient (\cref{prop:outinv,prop:quotient}). And logit
distillation is singled out as the objective already defined there
(\cref{sec:objectives}).

\begin{definition}[Representation equivalence class]
\label{def:equiv}
For $\HH \in \R^{N\times d}$ define
\begin{equation}
\label{eq:class}
[\HH] \;=\; \bigl\{\, c\,\HH \QQ \;:\; \QQ \in \Orth(d),\;
c \in \R_{+} \,\bigr\},
\end{equation}
the orbit of $\HH$ under right multiplication by orthogonal matrices and
\emph{isotropic} positive scaling. We restrict to isotropic scaling
deliberately: the invariants we use below (Gram structure, CKA) are invariant
to $\Orth(d)$ and to isotropic scaling, but \emph{not} to a non--isotropic
diagonal $\Lam$, since $\HH\Lam$ changes $\HH\HH^\top$. The relevant
identifiability group for these invariants is therefore $\Orth(d)\times\R_+$,
and we define $[\HH]$ to match it exactly.
\end{definition}

\paragraph{Student and teacher representations.}
The distillation setting compares two representations of the \emph{same} probe
set $\Probe$: the teacher's $\HH_T=\HH_T(\Probe)\in\R^{N\times d_T}$ and the
student's $\HH_S=\HH_S(\Probe)\in\R^{N\times d_S}$, each a chosen layer's
representation evaluated on the common probes. Two conditions are needed for an
absolute feature--matching loss $\norm{\HH_T-\HH_S}_F$ to be even well--posed as
a formula, and they are worth stating because they are exactly what the
equivalence--class view problematizes. First, the two must be evaluated on the
\emph{same} $\Probe$, so the rows correspond (a representation is only defined
relative to its probe set, \cref{sec:unident}); we assume this throughout.
Second, they must share dimension, $d_S=d_T=d$, so that the difference lives in a
single space $\R^{N\times d}$ and an orthogonal $\QQ\in\Orth(d)$ can act on
either. We take $d_S=d_T=d$ for \cref{prop:illposed} and the relational
constructions of \cref{sec:invariants}; the unequal--dimension case
$d_S\neq d_T$, where no common $\R^{N\times d}$ exists and a learned map must
bridge the two, is exactly what the alignment of \cref{sec:align}
(\cref{rem:unequal}) handles. Even with both conditions met, the next
proposition shows the loss is still ill--posed---not for lack of a shared space,
but because that space carries no absolute basis.

\begin{proposition}[Ill--posedness of absolute feature matching]
\label{prop:illposed}
Let $\HH_T,\HH_S\in\R^{N\times d}$ be teacher and student representations of a
common probe set $\Probe$ (same dimension $d$, as above), and consider the
feature--matching loss $\Labs(\HH_S) = \norm{\HH_T - \HH_S}_F^2$. Under
\cref{ass:absorb}, $\Labs$ is not constant on the student equivalence class
$[\HH_S]$: there exist $\QQ \in \Orth(d)$ with
$\Labs(\HH_S \QQ) \neq \Labs(\HH_S)$ while $\HH_S\QQ$, once completed by the
corresponding downstream reparametrization $\Win \mapsto \QQ^\top\Win$
(\cref{ass:absorb}), induces the same network function as $\HH_S$ with $\Win$.
Consequently $\Labs$ assigns different values to representations that are
functionally equivalent once their readers are accounted for, and its minimizer
is an artifact of the arbitrary representative chosen for $[\HH_S]$.
\end{proposition}

\begin{proof}
Functional invariance of $\HH_S \QQ$ is immediate from \cref{ass:absorb}.
For non--constancy, expand
\[
\Labs(\HH_S\QQ) = \norm{\HH_T}_F^2 + \norm{\HH_S}_F^2
- 2\,\tr\!\bigl(\QQ^\top \HH_S^\top \HH_T\bigr).
\]
Only the last term depends on $\QQ$. Let $\HH_S^\top \HH_T = U\Sigma V^\top$
be a singular value decomposition with $\Sigma \neq 0$ (generic teacher and
student). The map $\QQ \mapsto \tr(\QQ^\top U \Sigma V^\top)$ ranges over
$[-\sum_i \sigma_i, \sum_i \sigma_i]$ as $\QQ$ varies over $\Orth(d)$, taking
its maximum at $\QQ = U V^\top$ (so that $\QQ^\top U\Sigma V^\top = V\Sigma
V^\top \succeq 0$). Hence it is non--constant, so $\Labs$ is
non--constant on $[\HH_S]$.
\end{proof}

\paragraph{The proposition, measured.}
\Cref{prop:illposed} can be turned into a direct measurement. Take a real
teacher representation $\HH$ (Qwen2.5--0.5B, layer $12$, $d=896$) and apply an
orthogonal reparametrization $\HH\mapsto\HH\QQ$, interpolating $\QQ$ from the
identity to a random rotation. The rotation is function--preserving in the
sense of \cref{prop:outinv}: a compensating downstream reparametrization
$\phi^{\QQ}$ exists for which $f^{\QQ}=f$. The table evaluates the
distillation losses only on the two \emph{representatives} $\HH$ and $\HH\QQ$,
which is exactly what a representational objective sees. We compare three
losses on the pair $(\HH,\HH\QQ)$: absolute
feature matching $\norm{\HH-\HH\QQ}_F^2$, the CKA loss $1-\CKA$, and the
normalized--Gram loss $L_{\mathrm{rel}}$. The two representatives are the
\emph{same information in rotated coordinates}, yet \cref{tab:rotation} shows
the feature--matching loss growing without bound (to $\sim\!7\times10^5$ at a
full rotation) while $1-\CKA$ and $L_{\mathrm{rel}}$ stay at $0$ to machine
precision for every rotation. Absolute matching penalizes a change that alters
nothing about the function; the basis--invariant objectives are exactly blind
to it. This is \cref{prop:illposed} as an experimental fact; its
\emph{training--level} counterpart appears in \cref{sec:experiments}, where
optimizing $L_{\mathrm{cka}}$ to $\CKA\!\to\!1$ still leaves the function
destroyed (\cref{tab:ablation})---\cref{tab:rotation} is the static form, the
ablation the dynamic one.

This also settles a natural worry: is ``$\CKA\!\approx\!1$ yet capability zero''
an artifact of CKA specifically---a known insensitivity of CKA to
small--eigenvalue directions---rather than a fact about representation matching?
It is not. The failure is not a property of CKA as a \emph{measure} but of
basis invariance as a \emph{class} of objective. \Cref{prop:illposed} is proved
for any loss that is constant on $[\HH_S]$, and \cref{tab:rotation} exhibits the
same blindness for $L_{\mathrm{rel}}$ (a normalized--Gram, kernel--free
objective) as for CKA: both stay at $0$ under a full function--preserving
rotation. Any measure fine enough to distinguish the two representatives would,
by definition, no longer be basis--invariant---and would then penalize the very
rotations that leave the function unchanged, i.e.\ behave like absolute matching
and fail for the opposite reason. The gap between $\CKA\!\to\!1$ and capability
is therefore not a resolution limit of one similarity index but the structural
fact that the output--selecting representative is quotiented away by every
basis--invariant objective.

\begin{table}[t]
\centering
\caption{\Cref{prop:illposed}, measured on Qwen2.5--0.5B (layer 12, $d=896$).
A function--preserving rotation $\QQ$ is interpolated from identity ($t=0$) to
a random orthogonal matrix ($t=1$). Absolute feature matching diverges; the
basis--invariant losses are identically zero. The feature--matching column
reports $\Labs$ of \cref{prop:illposed} normalized per probe ($\,/N$).}
\label{tab:rotation}
\begin{tabular}{l ccc}
\toprule
rotation $t$ & $\norm{\HH-\HH\QQ}_F^2/N$ & $1-\CKA$ & $L_{\mathrm{rel}}$\\
\midrule
0.00 & $0$               & $0$ & $0$\\
0.10 & $1.1\times10^4$   & $0$ & $0$\\
0.25 & $6.9\times10^4$   & $0$ & $0$\\
0.50 & $2.5\times10^5$   & $0$ & $0$\\
0.75 & $4.9\times10^5$   & $0$ & $0$\\
1.00 & $7.0\times10^5$   & $0$ & $0$\\
\bottomrule
\end{tabular}
\end{table}

\section{Admissible Targets: Equivalence--Class Invariants}
\label{sec:invariants}

Before exhibiting specific quantities, we say precisely what it means for one to
be an invariant.

\begin{definition}[Class invariant]
\label{def:invariant}
A map $q$ defined on representations is a \emph{($\Grp$--)invariant} if it is
constant on each equivalence class, i.e.
\begin{equation}
\label{eq:invariant}
q(c\,\HH\QQ)=q(\HH)\qquad\text{for all }\QQ\in\Orth(d),\ c\in\R_+,
\end{equation}
equivalently if $q$ factors through the quotient onto
$\R^{N\times d}/\Grp$ (the representation quotient defined in
\cref{eq:pirep} below) as a well--defined function of the class. A function of
\emph{two} representations is an invariant
if it is constant under the action applied to either argument separately. An
invariant is thus a well--defined function of the class $[\HH]$, not of the
representative; it is exactly the kind of quantity a distillation objective may
legitimately target, since by \cref{prop:illposed} anything not of this form
depends on the arbitrary basis.
\end{definition}

We now exhibit three invariants: the Gram structure, $\CKA$, and the
principal--angle spectrum.

The first invariant is relational. The (uncentered) \emph{Gram matrix} is
$\GG = \HH \HH^\top \in \R^{N\times N}$, with centered version
$\widehat{\GG} = \bm{C}\GG\bm{C}$ where $\bm{C} = \bm{I} - \tfrac1N
\bm{1}\bm{1}^\top$ ($\bm{1}\in\R^{N}$ the all--ones vector, $\bm{I}$ the
$N\times N$ identity). It records inter--sample geometry, and is orthogonally
invariant:

\begin{lemma}[Orthogonal invariance of the Gram structure]
\label{lem:gram}
For all $\QQ \in \Orth(d)$, $(\HH\QQ)(\HH\QQ)^\top = \HH\HH^\top$. Thus $\GG$
and $\widehat{\GG}$ are invariant under the orthogonal part of $[\HH]$.
\end{lemma}

\begin{proof}
$(\HH\QQ)(\HH\QQ)^\top = \HH\QQ\QQ^\top\HH^\top = \HH\HH^\top$ since
$\QQ\QQ^\top = \bm{I}$.
\end{proof}

The Gram matrix records inter--sample geometry (distances and angles), which
is exactly the information the representation carries about the probe set,
stripped of the arbitrary coordinate frame.

\begin{definition}[Centered Kernel Alignment]
\label{def:cka}
For representations $\HH_A,\HH_B$ with centered Gram matrices
$\widehat{\GG}_A,\widehat{\GG}_B$, and writing
$\inner{A}{B}_F=\tr(A^\top B)$ for the Frobenius inner product (with
$\norm{A}_F=\inner{A}{A}_F^{1/2}$),
\begin{equation}
\label{eq:cka}
\CKA(\HH_A,\HH_B) \;=\;
\frac{\inner{\widehat{\GG}_A}{\widehat{\GG}_B}_F}
{\norm{\widehat{\GG}_A}_F\,\norm{\widehat{\GG}_B}_F}\;\in[0,1].
\end{equation}
Linear CKA was introduced by \citet{kornblith2019} as a representation
similarity index invariant to orthogonal transformation and isotropic
scaling, building on kernel--target alignment~\citep{cristianini2001} and
centered alignment~\citep{cortes2012}; it is a normalized
Hilbert--Schmidt independence criterion~\citep{gretton2005}.
\end{definition}

\begin{corollary}
$\CKA$ is invariant under orthogonal transformation and isotropic scaling of
either argument, and therefore is a well--defined similarity on equivalence
classes $[\HH_A],[\HH_B]$.
\end{corollary}

\begin{proof}
Orthogonal invariance follows from \cref{lem:gram}. Isotropic scaling
$\HH \mapsto c\HH$ scales $\widehat{\GG}$ by $c^2$, which cancels in the
normalized ratio.
\end{proof}

A third invariant measures shared subspace. If $\HH_A,\HH_B$ have orthonormal
column bases $\bm{U}_A,\bm{U}_B$ and $k=\min(\rank\HH_A,\rank\HH_B)$, the
\emph{principal angles} $\theta_1\le\dots\le\theta_k$ satisfy
$\cos\theta_i = \sigma_i(\bm{U}_A^\top \bm{U}_B)$ (the $i$th singular value of
$\bm{U}_A^\top\bm{U}_B$), and the \emph{effective shared dimension} is
$k_{\mathrm{eff}} = \sum_{i=1}^{k} \cos^2\theta_i$. The values
$\{\cos^2\theta_i\}$ quantify how much of the two feature dictionaries point in
common directions; they will drive the graft predictions in \cref{sec:graft}.

\paragraph{The output function is invariant under the joint class action.}
The invariants above (Gram, CKA, principal angles) are properties of the
representation. We now make explicit the link the rest of the paper relies on,
and to do so we name two distinct operations that the discussion repeatedly
contrasts.
\begin{itemize}[leftmargin=1.4em,itemsep=2pt,topsep=3pt]
\item A \emph{rotation (alone)} acts on the representation only:
$h\mapsto h\QQ$ with the downstream parameters $\theta$ left unchanged. This
\emph{does} change the network function in general---feeding $h\QQ$ into the
unchanged downstream produces different logits (\cref{rem:joint} below;
empirically $\norm{f_{\text{rot}}-f}\neq0$, \cref{tab:mlp}).
\item A \emph{joint reparametrization} acts on the representation \emph{and} the
downstream together: for $(\QQ,c)\in\Grp=\Orth(d)\times\R_+$,
$(h,\theta)\mapsto(c\,h\QQ,\ \theta^{\QQ,c})$, where $\theta^{\QQ,c}$ is the
absorbing change $\Win\mapsto c^{-1}\QQ^\top\Win$ of \cref{ass:absorb} (the
rotation $\QQ$ is undone by $\QQ^\top$ and the scale $c$ by $c^{-1}$). By
\cref{eq:absorb} this leaves the network function unchanged: $f^{\QQ,c}=f$.
\end{itemize}
The point of \cref{prop:outinv} is that the output function is invariant under
the \emph{joint} reparametrization, and \emph{not} under rotation alone. A
reparametrization of the representation does not act on $\HH_\ell$ in isolation:
under \cref{ass:absorb} the downstream blocks are reparametrized in
compensation, and it is this joint action---on the representation \emph{and} the
downstream parameters---that the function is constant under. This is exactly why
two networks differing by such an action implement the same function, and why
capability transfers along the class; it is also why a representational
objective, which sees only $h$ and can move it by a rotation alone, is not by
itself a capability objective.

\begin{proposition}[Output invariance under the joint class action]
\label{prop:outinv}
Write the network function as
$f(x)=\phi_{\ell{:}L}\!\bigl(h_\ell(x)\bigr)$, with $\phi_{\ell{:}L}:\R^d\to\R^V$
the map from the layer--$\ell$ representation $h_\ell(x)\in\R^d$ to the output
logits as in \cref{ass:absorb} (it carries the remaining blocks and the
unembedding). For $\QQ\in\Orth(d)$, let $\phi_{\ell{:}L}^{\QQ}$ denote
$\phi_{\ell{:}L}$ with its input operator $\Win$ reparametrized to absorb $\QQ$
(the existence of such a $\phi_{\ell{:}L}^{\QQ}$ is \cref{ass:absorb}). Then for
all $x$,
\begin{equation}
\label{eq:propinv}
\phi_{\ell{:}L}^{\QQ}\!\bigl(h_\ell(x)\,\QQ\bigr)=\phi_{\ell{:}L}\!\bigl(h_\ell(x)\bigr),
\qquad\text{hence}\qquad f^{\QQ}(x)=f(x),
\end{equation}
where $f^{\QQ}$ is the network with representation $h_\ell\QQ$ and downstream
$\phi_{\ell{:}L}^{\QQ}$. Thus the pair $(h_\ell,\ \theta)$, with $\theta$
the downstream parameters,
transforms under a joint equivalence action, and $f$ is invariant under it: the
output function---and any capability metric computed from the logits
(perplexity, top--1 agreement)---is constant along the orbit
$\{(c\,h_\ell\QQ,\ \phi_{\ell{:}L}^{\QQ,c}):(\QQ,c)\in\Grp\}$. The invariant
object is the orbit $[(h_\ell,\theta)]$; the pair $(h_\ell,\theta)$ is what the
action moves.
\end{proposition}

\begin{proof}
This is precisely \cref{eq:absorb}. The substitution $\Win\mapsto\QQ^\top\Win$
gives the cancellation~(\ref{eq:absorb}a),
$\tilde\psi\bigl((h_\ell\QQ)(\QQ^\top\Win)\bigr)=\tilde\psi(h_\ell\Win)$, so
every downstream activation, and thus the logits, are unchanged; hence
the consequence~(\ref{eq:absorb}b),
$\phi_{\ell{:}L}^{\QQ}(h_\ell\QQ)=\phi_{\ell{:}L}(h_\ell)$, and $f^{\QQ}=f$.
Invariance for all $\QQ\in\Orth(d)$ means $f$ is constant on the joint orbit.
\end{proof}

\begin{remark}
\label{rem:joint}
The invariance is \emph{joint}: it is not true that feeding $h_\ell\QQ$ into the
\emph{unchanged} $\phi_{\ell{:}L}$ preserves the output---that would generally
change $f$. What is preserved is the function realized by the
jointly--reparametrized network. Equivalently, the output function is a
well--defined function on the quotient by the joint action, not on $[h_\ell]$
alone. This is the precise sense in which capability is a class invariant.
An analogy: reaching the representation quotient is like walking into the right
meeting room (the correct geometric structure); reaching the joint quotient is
like sitting in the exact seat assigned to you (the coordinates the output head
can read). $\CKA$ gets you into the room, but it does not tell you which seat
is yours---which is why aligning the representation is not the same as restoring
the function.
\end{remark}

\begin{remark}[Orbit and fiber, in programmers' terms]
\label{rem:csorbit}
Two words organize what follows, and both have exact programming analogues. An
\emph{orbit} is all implementations of one program: rewriting the two--layer
network's internals by any $(\QQ,c)$---$h\mapsto c\,h\QQ$ with the reader updated
to $c^{-1}\QQ^\top\WW_2$---gives different source code computing the same output,
and the set of all such rewrites is the orbit of $(h,\WW_2)$. A \emph{fiber} is
everything mapping to one behavior: for a map $\pi$ and a fixed output $y$, the
fiber $\pi^{-1}(y)$ collects all inputs $\pi$ sends to $y$. Here $\pi$ is the
joint quotient map (defined next), and its fiber over a class is exactly one
orbit---the clones computing the same function. The two words name the same
subset; \cref{sec:code-appendix} realizes both as runnable code. Everything that
follows is bookkeeping for these sets: a legitimate distillation target may depend
on \emph{which} fiber the student lands in, never on \emph{where inside} it sits.
\end{remark}

\paragraph{The chain as a factorization through the quotient.}
\Cref{prop:outinv} is best read through quotients by the class action, of which
there are two, and distinguishing them is what makes the chain precise. Let
$\Grp=\Orth(d)\times\R_+$. It acts on representations alone by $h\mapsto c\,hQ$.
The \emph{quotient space} $\R^{N\times d}/\Grp$ is the set of equivalence classes
under this action, $\R^{N\times d}/\Grp=\{\,[h]:h\in\R^{N\times d}\,\}$ with
$[h]$ as in \cref{eq:class} (we use only the elementary theory of group actions
and quotient spaces; see, e.g., \citealp{lee2013}); collapsing each orbit to a
single point yields the \emph{representation quotient map}
\begin{equation}
\label{eq:pirep}
\pi_{\mathrm{rep}}:\ \R^{N\times d}\longrightarrow \R^{N\times d}/\Grp,
\qquad \pi_{\mathrm{rep}}(h)=[h],
\end{equation}
on which the invariants of \cref{sec:invariants} ($\CKA$, normalized Gram,
principal angles) are functions. But the network function does \emph{not} live
on $\R^{N\times d}/\Grp$: by \cref{rem:joint} feeding $h\QQ$ into the unchanged
downstream changes $f$. The function lives on the \emph{joint} quotient. Writing
$\Theta$ for the downstream parameters and letting $\Grp$ act jointly, for
$g=(\QQ,c)\in\Grp$,
$g\cdot(h,\theta)=(c\,h\QQ,\ \theta^{\QQ,c})$ with $\theta^{\QQ,c}$ the absorbing
reparametrization $\Win\mapsto c^{-1}\QQ^\top\Win$ of \cref{ass:absorb}, define
the \emph{joint space}
$\MM=\R^{N\times d}\times\Theta$ and its quotient $\MM/\Grp=\{\,[(h,\theta)]:
(h,\theta)\in\MM\,\}$, the set of joint orbits, with quotient map
\[
\pi_{\mathrm{joint}}:\ \MM\longrightarrow \MM/\Grp,
\qquad \pi_{\mathrm{joint}}(h,\theta)=[(h,\theta)].
\]
The invariants of \cref{sec:invariants} live on $\pi_{\mathrm{rep}}$; the output
function, as the next proposition shows, lives on $\pi_{\mathrm{joint}}$.

\begin{figure}[h]
\centering
\includegraphics[width=\linewidth]{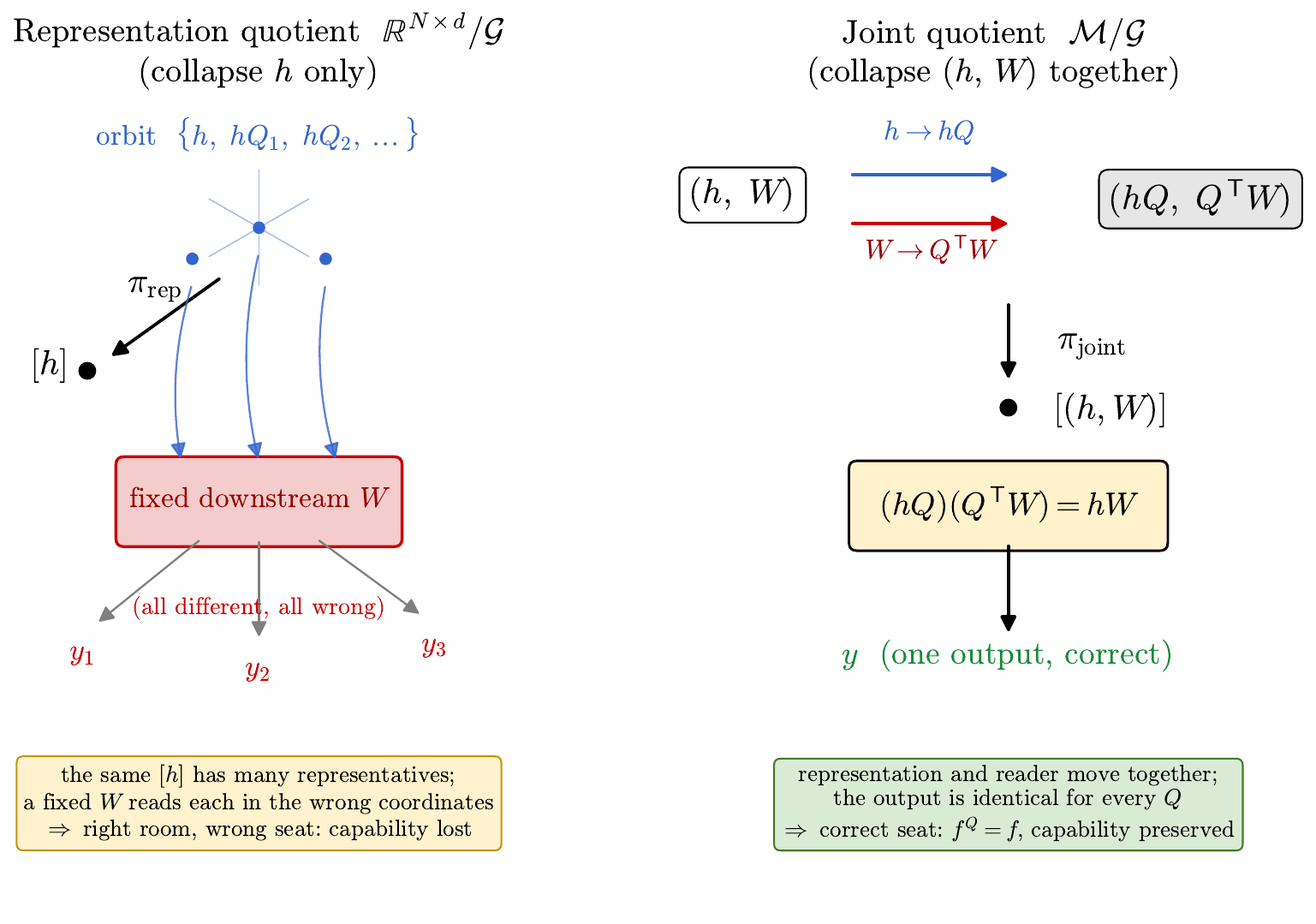}
\caption{Why capability lives on the joint quotient, not the representation
quotient (two--layer network of \cref{ex:quotient}). \emph{Left:} collapsing
$h$ alone to $[h]\in\R^{N\times d}/\Grp$ is what $\CKA$ does---but if the
downstream $W$ is held fixed, the rotated representation is read in the wrong
coordinates and the output is wrong, so matching $[h]$ does not restore
capability. \emph{Right:} the \emph{joint} action rotates $h$ and
counter--rotates $W$ together; the product $(h\QQ)(\QQ^\top W)=hW$ is unchanged,
so $(h,W)$ and its variants collapse to one point $[(h,W)]\in\MM/\Grp$ on which
the output---and hence capability---is constant ($f^{\QQ}=f$, \cref{eq:factor}).}
\label{fig:orbit}
\end{figure}

\begin{example}[Both quotients on a two--layer network]
\label{ex:quotient}
Concretely, on $x\to h\to y$ with $y=h\bm{W}_2$: rotating $h$ and
counter--rotating the reader, $(h\QQ)(\QQ^\top\bm{W}_2)=h\bm{W}_2=y$ (and
isotropic scale absorbed by $\bm{W}_2\mapsto c^{-1}\bm{W}_2$), leaves the output
fixed. So every $c\,h\QQ$ is the same point $[h]$ of the representation quotient
(\cref{fig:orbit}, left)---but $[h]$ alone forgets which $\bm{W}_2$ reads it
(\cref{rem:joint}). The pair $(h,\bm{W}_2)$ and its absorbable variants form one
point of the \emph{joint} quotient (\cref{fig:orbit}, right), and every network
there computes the same function. This is \cref{eq:factor} below.
\end{example}

\begin{proposition}[Factorization through the joint quotient]
\label{prop:quotient}
Intuitively, two networks differing only by an absorbable coordinate change are
treated as the same point, and the function reads only that point. Formally,
under \cref{ass:absorb} the output function descends to the joint quotient:
there is a map $\bar f$ on $\MM/\Grp$ with
\begin{equation}
\label{eq:factor}
f \;=\; \bar f\circ\pi_{\mathrm{joint}},
\qquad\text{i.e.}\qquad
f\ \text{depends on }(h_\ell,\theta)\ \text{only through }[(h_\ell,\theta)].
\end{equation}
Consequently any evaluation functional $\kappa$ computed from the logits
(perplexity, top--1 agreement, $\mathrm{KL}$ to a reference) factors as
$\kappa=\bar\kappa\circ\pi_{\mathrm{joint}}$: capability is constant on fibers
of $\pi_{\mathrm{joint}}$ (a \emph{fiber} is one orbit---all the clones that
compute the same function in different coordinates), hence a well--defined
function on the joint quotient.
It is \emph{not}, in general, a function on the representation quotient
$\R^{N\times d}/\Grp$ alone---that is exactly the content of \cref{rem:joint}.
\end{proposition}

\begin{proof}
The joint action $g\cdot(h,\theta)=(c\,h\QQ,\theta^{\QQ,c})$ of
$\Grp=\Orth(d)\times\R_+$ partitions $\MM$ into orbits, which are exactly the
fibers of $\pi_{\mathrm{joint}}$. By \cref{prop:outinv}, for every $g\in\Grp$ the
output is unchanged, $f(g\cdot(h,\theta))=f(h,\theta)$ (the absorbing
identity~(\ref{eq:absorb}), an algebraic equality), so \emph{$f$ is constant on
each fiber}. A map constant on fibers descends uniquely through the surjection
$\pi_{\mathrm{joint}}$ (the universal property of the quotient): setting
$\bar f([(h,\theta)]):=f(h,\theta)$ is well--defined by fiber--constancy and is
the only $\bar f$ with $\bar f\circ\pi_{\mathrm{joint}}=f$, which is
\eqref{eq:factor}. Any evaluation functional $\kappa$ of the logits is a function
of $f$, so $\kappa=\bar\kappa\circ\pi_{\mathrm{joint}}$ inherits the same
factorization. Every step is an equality over the subgroup on which
\cref{ass:absorb} is exact (\cref{app:status}).
\end{proof}

We read \cref{eq:factor} with the input held fixed: for each $x$, the pair
$(h_\ell(x),\theta)$ determines the logits, and $\bar f$ sends the joint class
$[(h_\ell(x),\theta)]$ to $f(x)\in\R^V$. The probe set $\Probe$ enters through
the representation $h_\ell(\cdot)$, so $\bar f$ is a function of the class at
each input rather than of a single global orbit; the invariance is over
$(\QQ,c)\in\Grp$, not over $x$.

The factorization is what makes the four operations of the paper commensurable:
a distillation objective is \emph{admissible} exactly when it factors through
the appropriate quotient---constant on fibers, hence a function on the quotient.
\Cref{prop:illposed} is the statement that absolute feature matching does
\emph{not} factor through $\pi_{\mathrm{rep}}$ (it separates points of a single
fiber), which is why its minimizer is fiber--dependent and meaningless. The
invariants of \cref{sec:invariants} are coordinates on the representation
quotient $\R^{N\times d}/\Grp$; the alignment of \cref{sec:align}
(\cref{def:align}) is a choice of
representative within a fiber; and---the point developed next---logit
distillation lives on the joint quotient $\MM/\Grp$ by construction.

\Cref{prop:quotient} is the theoretical hinge of the paper, and it settles the
loss taxonomy of \cref{sec:objectives} in one stroke. Because capability is
determined by $f=\bar f\circ\pi_{\mathrm{joint}}$, a function on the joint
quotient, an objective can transfer capability only if it targets the
output--function invariant on that quotient---being defined on the quotient is
necessary (a loss that separates points of one fiber penalizes
function--preserving reparametrizations), though not sufficient (a constant loss
is defined on the quotient yet transfers nothing). This
immediately sorts the candidate losses into two kinds. A loss written on the raw
representation $\R^{N\times d}$ is \emph{not} on the quotient as written. To be
admissible it must be \emph{pushed down} onto it---made constant on fibers. That
is precisely the normalization turning $\norm{\GG_S-\GG_T}$ into the
class--invariant $L_{\mathrm{rel}}$ and $L_{\CKA}$. A loss written on the output
distribution, by contrast, is \emph{already} a function on the joint quotient:
there is no fiber left to quotient out. This is why logit distillation needs no
such engineering and is the most stable objective. Every distinction in
\cref{sec:objectives} is a reading of this one fact: hidden--state losses must be
lowered to the quotient, the logit loss is born there.

This is the bridge between the two halves of the paper. The representational
invariants ($\CKA$, Gram, angles) say \emph{when two representations are the
same class}; \cref{prop:outinv,prop:quotient} say the \emph{output function is a
function on the quotient}, not a property of a chosen representative. It is
why \cref{prop:illposed} (absolute matching is ill--posed) and the empirical
result that output--function matching---not representational matching---restores
capability (\cref{sec:experiments}) are two sides of one fact: capability lives
on the quotient, so an objective can recover it only by targeting the
output--function invariant there.

\section{Alignment Across the Equivalence Class}
\label{sec:align}

To match coordinates (not just relations), one must first move the student
to a representative of the teacher's class. We make ``alignment'' precise as an
operation on the equivalence--class structure of \cref{sec:unident}, before
giving the map that realizes it.

\begin{definition}[Alignment]
\label{def:align}
Let $\HH_A$ and $\HH_B$ be representations of a common probe set $\Probe$, with
$[\HH_B]=\{\,c\,\HH_B\QQ:\QQ\in\Orth(d),\,c\in\R_+\,\}$ its equivalence class
(\cref{eq:class}). An \emph{alignment of $\HH_B$ to $\HH_A$} is the choice of the
representative of $[\HH_B]$ closest to $\HH_A$, i.e.\ the element
\begin{equation}
\label{eq:aligndef}
\HH_B^{\,\star} \;=\; \argmin_{\HH'\in[\HH_B]} \norm{\HH_A - \HH'}_F
\;=\; c^\star\,\HH_B\QQ^\star ,
\end{equation}
together with the group element $(\QQ^\star,c^\star)\in\Grp$ that realizes it.
Equivalently, alignment is a section of the representation quotient map
$\pi_{\mathrm{rep}}$ over the fiber $\pi_{\mathrm{rep}}^{-1}([\HH_B])$: it selects
one point of the fiber (one coordinate frame) rather than collapsing the fiber
to the class. It is thus dual to the invariants of \cref{sec:invariants}, which
quotient the fiber out; alignment instead fixes a representative within it. The
attained value $\rho=\norm{\HH_A-\HH_B^{\,\star}}_F/\norm{\HH_A}_F$ is the
\emph{alignment residual}, a class invariant of the pair (constant under
separate reparametrizations of $\HH_A,\HH_B$) that measures how far the two
classes are from coinciding. When $d_A\ne d_B$ no single fiber contains both, and
$\QQ^\star$ is replaced by a learned low--capacity map $\WW$ (\cref{rem:unequal})
that plays the same role of selecting coordinates.
\end{definition}

\Cref{prop:proc} solves \cref{eq:aligndef} in closed form for the orthogonal
part of the group (the scale $c^\star$ is fixed separately by
$c^\star=\inner{\HH_A}{\HH_B\QQ^\star}/\norm{\HH_B}_F^2$).

\begin{proposition}[Orthogonal Procrustes alignment]
\label{prop:proc}
For $d_A=d_B=d$, the orthogonal rotation $\QQ^\star$ that aligns $\HH_B$ to
$\HH_A$ most closely is given in closed form by \eqref{eq:procrustes}: writing the
singular value decomposition $\HH_B^\top\HH_A=U\Sigma V^\top$, the minimizer is
$\QQ^\star=UV^\top$.
\begin{equation}
\label{eq:procrustes}
\QQ^\star = \argmin_{\QQ\in\Orth(d)} \norm{\HH_A - \HH_B\QQ}_F
= U V^\top,\qquad \HH_B^\top \HH_A = U\Sigma V^\top.
\end{equation}
The aligned residual
$\rho = \norm{\HH_A - \HH_B\QQ^\star}_F / \norm{\HH_A}_F$
is invariant under separate orthogonal reparametrizations of $\HH_A,\HH_B$ and
measures how far the two representations are from being the same information
in different coordinates. This is the orthogonal Procrustes
problem~\citep{schonemann1966}, used in representational similarity
analysis~\citep{kriegeskorte2008} to compare neural population codes.
\end{proposition}

\begin{proof}
Minimizing $\norm{\HH_A-\HH_B\QQ}_F^2$ over $\Orth(d)$ is equivalent to
maximizing $\tr(\QQ^\top \HH_B^\top\HH_A)$, whose maximizer is $UV^\top$ by
the von~Neumann trace inequality (saturated when $\QQ^\top U\Sigma V^\top$ is
symmetric positive semidefinite). Reparametrizing
$\HH_A\mapsto\HH_A R_A$, $\HH_B\mapsto\HH_B R_B$ replaces $\QQ^\star$ by
$R_B^\top \QQ^\star R_A$ and leaves the residual norm unchanged, giving
invariance of $\rho$.
\end{proof}

\begin{remark}[Unequal dimensions]
\label{rem:unequal}
\Cref{prop:proc} assumes teacher and student share a hidden dimension
($d_A=d_B$), so a rotation $\QQ\in\Orth(d)$ can align their coordinate frames
directly. In practice the two often differ---e.g.\ a $4096$--wide teacher and a
$2048$--wide student---and then no orthogonal $\QQ$ even exists, since rotation
is defined only within a single dimension. For $d_A\ne d_B$ we therefore replace
$\QQ$ by a learned linear map $\WW\in\R^{d_B\times d_A}$ solving
$\min_{\WW}\norm{\HH_A - \HH_B\WW}_F$ (ridge--regularized), the linear case of
feature distillation. The map should be \emph{low capacity}: alignment is
coordinate correction, not new learning. In other words, $\WW$ plays the role of
a coordinate converter, not a knowledge generator---its job is to express the
student's existing representation in the teacher's coordinate system, not to
synthesize structure the student does not have. A high--capacity $\WW$ (a deep
MLP, say) defeats this: it can manufacture the teacher's structure from an
unrelated student, after which one can no longer tell whether the two
representations were genuinely the same class or were forced to match by the
bridging network---the comparison stops measuring alignment and starts measuring
the capacity of $\WW$.
\end{remark}

\section{Basis--Invariant Distillation Objectives}
\label{sec:objectives}

Distillation estimates the student parameters $\theta_S$; these induce a
representation $\HH_S(\Probe)$ and an output function $f_S=f(\cdot;\theta_S)$, and
the pair $(\HH_S,\theta_S)$ determines the student's joint class. The question of
this section is which \emph{loss} to minimize in that estimation.
\Cref{prop:quotient} immediately induces the following taxonomy: since a loss is
admissible only if it lives on a quotient, every candidate falls into three
families by \emph{which} quotient that is (\cref{tab:losses}).

\begin{table}[h]
\centering
\small
\begin{tabular}{@{}p{2.55cm}p{2.4cm}p{2.7cm}p{1.7cm}p{2.2cm}@{}}
\toprule
\textbf{Family} & \textbf{Target} & \textbf{Defined on} &
\textbf{Alignment?} & \textbf{Recovers capability?}\\
\midrule
(A) Relational \newline ($L_{\mathrm{rel}},L_{\CKA}$) &
representation geometry (Gram) &
representation quotient $\R^{N\times d}/\Grp$ &
no &
no (structure only)\\[2pt]
(B) Align--then--match &
aligned absolute coordinates &
a chosen representative &
yes (Procrustes / $\WW$) &
conditional (on alignment)\\[2pt]
(C) Logit KD \newline ($L_{\mathrm{logit}}$) &
output distribution &
joint quotient $\MM/\Grp$ &
none needed &
yes (targets the invariant)\\
\bottomrule
\end{tabular}
\caption{The three basis--invariant families, by where each loss lives relative
to the quotient. (A) and (B) act on the hidden representation and must be made
basis--free by hand; (C) is already a function on the joint quotient, which is
why it alone targets capability directly (\cref{prop:quotient}).}
\label{tab:losses}
\end{table}

\emph{Only the third family targets capability without requiring an additional
coordinate--selection or reader--compatibility assumption}, because capability
lives on the joint quotient (\cref{prop:quotient}); the other two live on the
coarser representation quotient and, on their own, transfer representation
geometry rather than the output function. This is not to say a feature objective
\emph{cannot} contribute to capability---family~(B), if it also learns or fixes
the downstream reader so that the aligned frame is the one the reader expects,
can---only that doing so reintroduces exactly the coordinate--selection step the
logit objective never needs. The three are
exhaustive because there are only three things a basis--invariant loss can do
about the arbitrary coordinate frame: discard it (A), fix it first (B), or never
read it (C)---and only the last is born on the quotient where capability lives.
The rest of this section develops each in turn, then states the definition of
admissibility that makes ``defined on the quotient'' precise.

We use the notation already established: $\HH_S,\HH_T\in\R^{N\times d}$ are the
student and teacher representations of a common probe set $\Probe$
(\cref{sec:unident}); $\GG=\HH\HH^\top$ is the Gram matrix and
$\widehat{\GG}=\bm{C}\GG\bm{C}$ its centered form, with $\CKA$ the normalized
similarity built from them (\cref{def:cka}); $\WW$ is the low--capacity
alignment map of \cref{rem:unequal}.

\paragraph{(A) Relational matching (alignment--free).}
The first option throws the basis away entirely and compares only what survives
it---the geometry of the representation, i.e.\ the angles and distances between
the $N$ probe points, which the Gram matrix records. Two representations of the
same class have the same geometry, so matching geometry needs no alignment. To
be invariant to the full $\Orth(d)\times\R_+$ class, the Gram matrices must be
normalized before comparison:
\begin{equation}
\label{eq:lrel}
L_{\mathrm{rel}} = \norm{\frac{\widehat{\GG}_S}{\norm{\widehat{\GG}_S}_F}
- \frac{\widehat{\GG}_T}{\norm{\widehat{\GG}_T}_F}}_F^2
\quad\text{or}\quad
L_{\CKA} = 1 - \CKA(\HH_S,\HH_T).
\end{equation}
Both are invariant to orthogonal transformation \emph{and} isotropic scaling.
We emphasize that the \emph{unnormalized} Gram loss
$\norm{\widehat{\GG}_S - \widehat{\GG}_T}_F^2$ is orthogonal--invariant but
\emph{not} scale--invariant---under $\HH\mapsto c\HH$ the Gram scales as $c^2$,
so it is not constant on $[\HH]$ and would reintroduce a (milder) version of
the ill--posedness of \cref{prop:illposed}. Normalization (or, equivalently,
the cosine form $L_{\CKA}$) is required. No alignment is estimated, and the two
arguments need not share dimension (both Gram matrices are $N\times N$); this is
the loss to reach for when teacher and student have different widths. The
cost is that absolute position is not transferred, which matters when a
downstream module consumes the representation coordinatewise.

\paragraph{(B) Align--then--match.}
The second option keeps the basis but fixes it first: estimate the rotation
(or low--capacity map) that moves the student onto a representative of the
teacher's class, then match coordinates in that aligned frame. This is the
right choice when absolute position \emph{is} needed---when a module will read
the features coordinatewise, as in grafting.
\begin{equation}
\label{eq:lalign}
L_{\mathrm{align}} = \norm{\HH_S\WW - \HH_T}_F^2,
\qquad \WW\ \text{low capacity (\cref{prop:proc} / its remark).}
\end{equation}
This transfers absolute coordinates and is the appropriate objective whenever a
downstream consumer reads the features in a specific frame: model grafting and
activation transplants (\cref{sec:graft}), intermediate--layer supervision
(FitNets--style hints), and cross--width distillation through a learned projection
all fall here. The price is injecting alignment--estimation noise, and its
recovery of capability is only \emph{conditional}---good exactly when the
estimated alignment is good.

\paragraph{(C) Output matching: the objective already living on the
quotient.}
The third option sidesteps the basis problem by never looking at the
representation at all---it matches only the teacher's output distribution. Since
two networks of the same class produce the same outputs (\cref{prop:quotient}),
output matching is automatically basis--free; there is nothing to normalize or
align away, because the loss never sees a basis in the first place. This is
standard logit distillation: writing $\bm{z}_T(x),\bm{z}_S(x)\in\R^V$ for the
teacher and student output logits at input $x$ (the student's being
$\bm{z}_S(x)=f_S(x;\theta_S)$, computed from the estimated parameters
$\theta_S$) and $p_\tau(\bm{z})=\mathrm{softmax}(\bm{z}/\tau)$ for the
temperature--$\tau$ softmax, it is the full--sequence Kullback--Leibler
divergence between the two output distributions,
\begin{equation}
\label{eq:llogit}
L_{\mathrm{logit}}
= \frac{1}{\lvert\Probe\rvert}\sum_{x\in\Probe}
  \mathrm{KL}\!\left(p_\tau\!\left(\bm{z}_T(x)\right)\,\middle\|\,
  p_\tau\!\left(\bm{z}_S(x)\right)\right),
\end{equation}
with $\tau=1$ in our experiments. We use $L_{\mathrm{logit}}$, \emph{logit--KD},
and \emph{logit distillation} interchangeably. The sum runs over the full
sequence, not the last token alone: matching only the final position drives the
last--token KL to $0.03$ while leaving the full--sequence KL at $\approx 6$ and
the perplexity ratio above $700\times$ (\cref{res:ablation}).

The point is \emph{where} logit--KD is defined. Families (A) and (B) are hidden
losses that must be \emph{engineered} onto a quotient---(A) discards the basis,
(B) fixes it first---because a raw function on $\R^{N\times d}$ is not constant on
fibers (\cref{prop:illposed}). Logit--KD never reads $\R^{N\times d}$ at all: by
\cref{prop:quotient} the output distribution is $\bar f\circ\pi_{\mathrm{joint}}$,
a function on the joint quotient by construction. It is not ``invariant by
luck''---it is \emph{natively defined on the quotient where the function lives},
with no fiber to be made constant over.

This settles the ordering question when teacher and student differ in width
($d_A\ne d_B$): a \emph{feature} objective must align first ($\WW$ then match,
family (B)), but logit--KD needs no such step---its arguments are the teacher and
student logits in the shared $V$--dimensional vocabulary, so the mismatched hidden
widths never meet. The alignment machinery of \cref{sec:align} is needed exactly
when one transfers \emph{coordinates} (feature matching, grafting), and is
superfluous when one only wants the function. The price is that logit--KD
transfers no intermediate structure---so when a downstream module must consume
features coordinatewise (\cref{sec:graft}), one still wants family (A) or (B)
alongside it. In that regime, feature--level distillation matching intermediate
activations (FitNets~\citep{romero2015}) needs the alignment of \cref{sec:align},
while relation--based distillation~\citep{park2019,tung2019} matches
basis--invariant structure and avoids it. A natural combined objective is
\begin{equation}
\label{eq:combined}
\Obj = L_{\mathrm{logit}} + \lambda\,L_{\mathrm{rel}},
\end{equation}
pairing the stability of output matching with the structural content of
relational matching.

\paragraph{Why not a cosine or MSE feature loss?}
A natural objection is that simpler feature losses---squared error
$\norm{\HH_S-\HH_T}_F^2$, or a row--wise cosine loss $1-\cos(\HH_S,\HH_T)$---should
also work. They do not escape \cref{prop:illposed}. The MSE loss is exactly the
absolute feature matching the proposition rules out: it is not constant on
$[\HH_T]$, so its minimizer depends on the arbitrary representative. The cosine
loss removes the scale $c$ and per--row magnitude, but \emph{not} a shared
orthogonal $\QQ$: rotating every row of $\HH_T$ by the same $\QQ$ changes each
row's direction while preserving the function, so cosine still penalizes
function--preserving reparametrizations and is not a class invariant. To make a
feature loss admissible one must quotient out $\Grp$ entirely---either compare
the basis--invariant Gram/$\CKA$ (family A), which discards the basis, or first
estimate the alignment (family B), which fixes it. Cosine fixes the scale but
not the rotation, so it sits strictly between raw MSE and a basis--invariant loss
without reaching the latter.

\begin{figure}[h]
\centering
\begin{tikzpicture}[
  >={Stealth[length=2mm]},
  node distance=6mm,
  box/.style={draw, rounded corners, align=center, inner sep=4pt,
              font=\small, minimum height=8mm},
  hdr/.style={box, fill=black!4},
  every edge/.style={draw, ->, thick}
]
\node[hdr] (cls) {Teacher equivalence class $[\HH_T]$};
\node[box, below left=11mm and 3mm of cls] (geo) {Representation geometry\\ {\footnotesize(Gram, $\CKA$, angles)}};
\node[box, below right=11mm and 3mm of cls] (out) {Output function\\ {\footnotesize(logits)}};
\node[box, below=12mm of geo] (rel) {$L_{\mathrm{rel}},\,L_{\CKA}$};
\node[box, below=12mm of out] (log) {$L_{\mathrm{logit}}$};

\draw (cls) edge (geo);
\draw (cls) edge (out);
\draw (geo) edge node[font=\footnotesize\itshape, left] {match} (rel);
\draw (out) edge node[font=\footnotesize\itshape, right] {match} (log);
\end{tikzpicture}
\caption{The paper in one figure. A teacher representation is an equivalence
class, not a fixed matrix, and it exposes two admissible targets: its
representation geometry (a basis--invariant relational structure) and its output
function (the logits, a class invariant by \cref{prop:quotient}). Relational KD
matches the geometry; logit KD matches the output function, which is why only the
latter transfers capability. Absolute feature matching---an arrow drawn straight
to $\HH_S$---is absent by design: it is the one target \cref{prop:illposed}
forbids. The recommended objective $\Obj=L_{\mathrm{logit}}+\lambda L_{\mathrm{rel}}$
(\cref{eq:combined}) uses both.}
\label{fig:summary}
\end{figure}
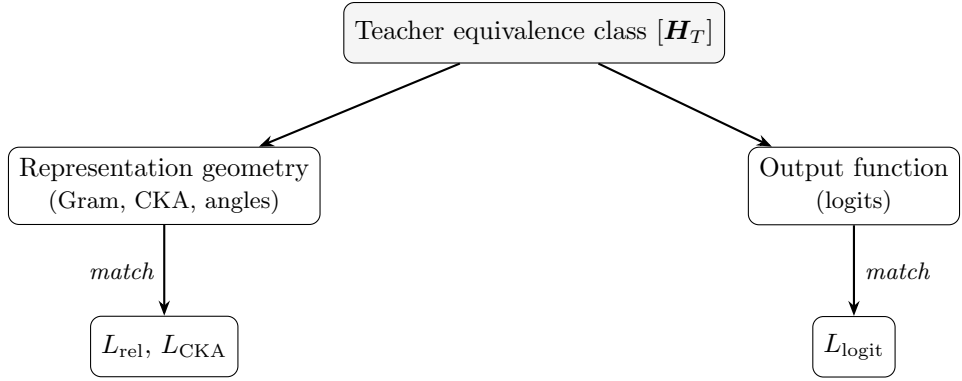

\Cref{fig:summary} is the whole argument at a glance: because the teacher is a
class, one does not match $\HH_T$ to the student directly; one matches what the
class actually determines---its representation geometry, its output function, or
both.

Having seen the three families concretely, we state the criterion that separates
them.

\begin{definition}[Admissible distillation objective]
\label{def:distill}
An objective is \emph{admissible} if it depends on the teacher only \emph{through
a quotient of the class action}---it factors as $\Obj=\bar\Obj\circ\pi$ for some
$\pi\in\{\pi_{\mathrm{rep}},\pi_{\mathrm{joint}}\}$, hence is constant on the
fibers $\pi$ collapses. Its \emph{level} is the coarsest such quotient: families
(A) and (B) sit at the representation level ($\pi_{\mathrm{rep}}$), transferring
representation geometry; family (C) sits at the joint level ($\pi_{\mathrm{joint}}$),
a function of $f_S$ alone, which is why it transfers capability. Absolute feature
matching is \emph{not} admissible (\cref{prop:illposed}): it separates points of a
single fiber. Since the invariants are relative to $\Probe$, what any objective
transfers is a function of \emph{teacher $\times\,\Probe$}, bounded above by the
teacher ($\HH_S\preceq\HH_T$).
\end{definition}

\section{Predicting Graft Success from Subspace Overlap}
\label{sec:graft}

The same equivalence--class view that organized distillation predicts when
\emph{module transfer} works. Grafting takes a block from one model and drops it
into another; sometimes it works, sometimes it destroys the network, and prior
work has largely left the difference unexplained. The class view supplies the
missing variable. A grafted block is read by the host's downstream layers, and
those layers expect features in a particular subspace---the donor's block speaks
in its own basis, and the host can use it only to the extent the two share that
subspace. The intuition is linguistic: two English speakers understand each
other; an English and a Chinese speaker, handed the same word, do not. A layer's
output is a feature dictionary, and the next layer can read it only if the
dictionaries overlap. We therefore expect graft success to be governed not by
absolute features but by how much the donor and host representations \emph{share
a subspace} at the seam---a basis--invariant quantity, exactly the kind
\cref{sec:invariants} identified.

We fix terminology. \emph{Grafting} inserts a module (one or more layers) taken
from a \emph{donor} model into a \emph{host} model in place of the host's own,
a setting studied empirically as model (or network) stitching
\citep{lenc2015,bansal2021} and closely related to the permutation and
linear--mode--connectivity symmetries of trained networks
\citep{entezari2022,ainsworth2023}. We deliberately write \emph{donor} and
\emph{host} rather than teacher and student: grafting is not knowledge
transfer from one model to another (there is no $S\preceq T$ relation and no
objective being minimized) but a \emph{part transplant}, in which the host keeps
most of its body and swaps in one block from the donor. The two roles are
asymmetric in a way teacher/student is not---the host supplies the surrounding
network, the donor a single module---so the stitching vocabulary fits better.

\begin{definition}[Seam]
\label{def:seam}
The \emph{seam} is the interface between the donor module and the host network at
which the donor--side representation $\HH_D$ becomes the input read by the host's
downstream layers, whose own boundary representation is $\HH_H$. Both $\HH_D$ and
$\HH_H$ are taken on a common probe set at this interface.
\end{definition}

A graft \emph{succeeds} when it preserves the network's function (operationally,
when the grafted model's perplexity stays close to the host's; we use
$\mathrm{PPL}_{\text{graft}}/\mathrm{PPL}_{\text{host}}<2$ in the experiment). We
posit that this success is governed by the overlap of the two feature subspaces
at the seam.

\begin{assumption}[Overlap governs transfer]
\label{ass:overlap}
The probability that a graft preserves function depends on $(\HH_D,\HH_H)$
only through equivalence--class invariants of their joint configuration,
monotonically in the degree of shared subspace.
\end{assumption}

Candidate scalar invariants are $\CKA(\HH_D,\HH_H)$, the Procrustes residual
$\rho$ of \cref{prop:proc}, and $k_{\mathrm{eff}}$. Here donor and host carry
\emph{independent} bases (each is identified only up to its own $\Grp$--action),
so the relevant invariance is the separate--argument one of
\cref{def:invariant}---constant under the action applied to either argument on
its own---not the joint action of \cref{prop:outinv}, which couples a single
representation to its own reader. All three candidates have this property.

\begin{prediction}[Monotonicity]
\label{pred:mono}
Graft success rate increases monotonically with $\CKA(\HH_D,\HH_H)$ and
decreases monotonically with the Procrustes residual $\rho$.
\end{prediction}

\begin{prediction}[Domain effect is mediated by overlap]
\label{pred:mediate}
Pairs pretrained on matched domains exhibit higher boundary CKA, and the
effect of domain match on graft success is mediated by CKA: conditioning on
CKA, the residual direct effect of domain match is small. The claim is causal in
structure---domain match does not help grafting \emph{directly}; it helps
because matched domains produce more overlapping representations
($\text{domain}\to\text{overlap}\to\text{success}$), and it is the overlap, not
the domain label, that the graft responds to. This is the
mechanism behind the empirical finding that domain alignment is the decisive
variable.
\end{prediction}

\begin{prediction}[Alignment helps most under mismatch]
\label{pred:align}
Inserting a low--capacity alignment map $\WW$ (\cref{sec:align}) before the
graft raises success most for low--overlap pairs and least for high--overlap
pairs, i.e.\ there is a negative interaction between baseline overlap and the
benefit of explicit alignment. High--overlap pairs are already aligned, so
there is little left to correct.
\end{prediction}

\paragraph{Empirical protocol.}
From grafting runs, collect for each donor--host pair the boundary
representations $(\HH_D,\HH_H)$ on a fixed probe set together with a
success/failure label. Then:
(1) compute $\CKA$, $\rho$, $k_{\mathrm{eff}}$;
(2) regress success on each invariant to test \cref{pred:mono};
(3) run a mediation analysis $\text{domain}\to\CKA\to\text{success}$ for
\cref{pred:mediate};
(4) ablate the alignment map $\WW$ and test the overlap$\times$alignment
interaction for \cref{pred:align}.

\paragraph{Controlled validation (Qwen2.5--0.5B).}
We instantiate this protocol with a layer--transplant experiment. Donor
variants are produced by perturbing chosen layers of the host to varying
strengths, spreading boundary overlap across $\CKA\in[0.31,1.00]$; each
donor's perturbed layer block is transplanted into a fresh copy of the host,
with \emph{success} as defined above ($\mathrm{PPL}_{\text{graft}}
/\mathrm{PPL}_{\text{host}}<2$). Over $135$ grafts (success rate $0.64$),
boundary overlap predicts success as \cref{pred:mono} requires: $\CKA$ enters
the success logit with coefficient $+60$ (AUC $0.79$) and the Procrustes
residual $\rho$ with coefficient $-9.5$ (AUC $0.79$); the continuous quality
$-\log(\mathrm{PPL}\text{ ratio})$ correlates $+0.59$ with $\CKA$ and $-0.76$
with $\rho$.

The experiment also sharpens the claim. Overlap is \emph{necessary but not
sufficient}: low--overlap grafts almost always fail, but high--overlap grafts
near the output (e.g.\ transplanting layers $20$--$21$ of $24$) fail even at
$\CKA\approx1.00$, because layers close to the readout amplify small mismatches
into large output changes. Graft success is thus governed by boundary overlap
\emph{together with} transplant depth---overlap sets the ceiling, depth governs
sensitivity. This is consistent with \cref{ass:overlap} as a monotone
(not deterministic) relationship.

We note the scope of this evidence. The donors here are controlled
perturbations of the host, so the study validates the overlap--predicts--success
relationship (\cref{pred:mono}) within a single model family; it does not by
itself establish the cross--model mediation of \cref{pred:mediate}, which would
require genuinely distinct pretrained donors. We therefore present
\cref{pred:mediate} as a prediction motivated by the controlled result and by
the prior observation that domain alignment is decisive in practice, and leave
its cross--model test to future work.

\section{Experiments: Restoration and Distillation on Qwen2.5}
\label{sec:experiments}

We test the theory directly, beginning with a controlled \emph{restoration}
experiment. The question is sharp: if we damage a pretrained model and then
re-train it
using only a basis--invariant signal, does the model's pretraining---and its
capability---come back? The final result (\cref{res:fmonly}) then moves beyond
restoration to genuine cross--width distillation between independently
pretrained models, and finds the same ordering.

\paragraph{Scope.} Our central claims are theoretical---the ill--posedness of
absolute matching and the identification of admissible invariant targets hold for
any architecture satisfying \cref{ass:absorb}. The experiments below are
\emph{controlled} validations on Qwen2.5 and Llama--3.1, intended to test the
theory's consequences rather than to demonstrate a production distillation
system or large--scale cross--model transfer; the cross--width distillation of
\cref{res:fmonly} is the one genuine (non--restoration) transfer we run, and it
remains modest in scale. Where a claim (e.g.\ cross--model graft
mediation) exceeds what the controlled setup tests, we say so explicitly.

\subsection{Setup}

All code, data, and the exact per--result mapping used below are available at
\url{https://github.com/MachineLearningHan/Invariant_Pgm}; each script carries a
comment naming the result it produces, and the repository's index links every
result and table in this section to the script that generates it.

The primary teacher $T$ is Qwen2.5--0.5B~\citep{qwen25} (24 decoder layers,
tied input/output embeddings), held frozen. A student
$S$ is obtained by \emph{corrupting} $T$: we re-initialize a fraction $s$ of
the middle decoder layers (protecting the first and last two, since tied
embeddings make the output head sensitive to embedding damage). Stage~1
trains $S$ to match $T$ with a basis--invariant per--layer objective; no label
supervision and no teacher logits are used in the pure basis--invariant runs.
We compare three objectives: $L_{\mathrm{cka}}$ and $L_{\mathrm{rel}}$
(basis--invariant, \cref{sec:objectives}), and $\Labs$ (the
ill--posed coordinate--matching control of \cref{prop:illposed}). All Stage~1
training matches \emph{last--token} pooled representations across all layers
on a generic text corpus.

To probe scale and, crucially, the role of weight tying, we repeat the
experiment on two untied models from different families:
Qwen2.5--7B--Instruct~\citep{qwen25} (28 decoder layers, hidden width $3584$)
and Llama--3.1--8B--Instruct~\citep{llama3} (32 decoder layers, hidden width
$4096$), both of which \emph{untie} the input embedding from the output head.
For these runs we train only the corrupted layers (and, where noted, the
output head) to keep the optimizer state tractable; this restriction makes the
tying distinction observable, as shown below. Using two distinct families lets
us check that the tying effect is not an artifact of a single architecture.

\paragraph{What this design is, and is not.} A word on scope, since the setup
governs how the numbers below should be read. The student is not an independent
model: it is the teacher with a fraction of its layers re--initialized, so it
shares the teacher's architecture, width, tokenizer, and---crucially---the
\emph{uncorrupted} layers verbatim. Those surviving layers keep the student near
the teacher's own representative of the class: the two start in a largely
\emph{shared coordinate frame}, not in the independent frames that two separately
pretrained models would occupy. This is deliberate. The experiment is a
controlled probe of the theory's \emph{mechanism}---it isolates what drives
capability (the output function, \cref{prop:quotient}) from what merely aligns
geometry (a representational invariant)---by holding the coordinate system fixed
so that the two objectives can be compared on equal footing. It is therefore a
\emph{damage--and--restore} study, not a cross--model distillation benchmark. The
harder case, where teacher and student are independently trained and occupy
different representatives (or differ in width, \cref{rem:unequal}), is exactly
where the alignment of \cref{sec:align} becomes necessary rather than incidental;
our controlled setup does not test it, and we do not claim it (\cref{sec:intro}).
What the study does establish---cleanly, because the frame is shared---is that
even with the representation held recoverable, capability returns only through
the output--function term, which is the proposition it was built to test.

We measure restoration at two levels. \emph{Representational}: per--layer
$\CKA(S,T)$ on held--out probes, judged against a permutation null.
\emph{Functional (capability)}: $\mathrm{KL}(T\Vert S)$ over next--token
distributions, perplexity ratio $\mathrm{PPL}_S/\mathrm{PPL}_T$, and top--1
next--token agreement. Probes are split into an \emph{in--domain} set (the
language/domain of the training corpus) and an \emph{out--domain} set
(a disjoint domain), to test whether restoration is confined to the subspace
the corpus spans.

\subsection{Result 1: Basis--invariant training restores representation,
not capability}

With $L_{\mathrm{cka}}$ alone, per--layer $\CKA(S,T)$ recovers from the
corrupted baseline ($\approx 0.90$) to $0.99$--$1.00$ in--domain across all
corruption strengths---the representational geometry is restored from a
label--free, basis--invariant signal. Yet the \emph{function} is not:
\cref{tab:cap} (left block) shows $\mathrm{KL}(T\Vert S)\approx 7$\,--\,$11$,
perplexity ratios of $10^2$\,--\,$10^4$, and near--zero top--1 agreement.
A model can have $\CKA=0.99$ and still be a different function. This is the
empirical face of the equivalence--class structure: $L_{\mathrm{cka}}$ drives
$S$ close to $[H_T]$ on the probe set (high $\CKA$) but does not select the
representative the (tied) output head
reads, so the logits---and the capability---remain wrong.

\subsection{Result 2: Capability is driven by the logit term, not by
representational matching}
\label{res:ablation}

Adding a full--sequence logit--distillation term
$\Obj=\lambda_{\ell}L_{\mathrm{logit}}+\lambda_r L_{\mathrm{cka}}$ recovers
capability: in--domain perplexity ratio $1.01\times$, top--1 $0.98$;
out--domain $1.06\times$, $0.72$ (\cref{tab:cap}, right). The logit term must
be applied over the full sequence---matching only the last token drives the
last--token KL to $0.03$ during training yet leaves the full--sequence KL at
$\approx 6$ and perplexity $>700\times$.

An ablation isolates the source of the recovery, and the result is sharper
than ``both terms help.'' On the hardest setting---\emph{every} decoder layer
re--initialized (Result~6) with the corpus widened so coverage is not the
bottleneck---we compare $L_{\mathrm{cka}}$ alone, $L_{\mathrm{logit}}$ alone,
and the sum (\cref{tab:ablation}). Representational matching alone drives
$\CKA$ to $0.999$ but leaves capability destroyed (perplexity ratio
$\sim\!10^6$, top--1 $0$). The logit term alone---which never references the
hidden representation---recovers capability essentially perfectly (ratio
$1.0\times$, top--1 $1.00$) while leaving $\CKA$ at only $0.89$. The two are
near--orthogonal: $L_{\mathrm{cka}}$ aligns representational geometry,
$L_{\mathrm{logit}}$ restores the function, and the function does not follow
from the geometry. This is the training--level form of the rotation result of
\cref{tab:rotation}: there, a function--preserving reparametrization moved the
feature--matching loss while leaving the invariants fixed; here, optimizing the
invariant ($L_{\mathrm{cka}}$) moves the representation into place while leaving
the function fixed---wrong, in the same way and for the same reason.

We read this as the central empirical lesson rather than a negative result.
It is the experimental face of the theory: \emph{capability is a property of
the output function, which is an invariant of the equivalence class, and
$L_{\mathrm{logit}}$ targets exactly that invariant} (\cref{sec:objectives}).
A representational objective such as $L_{\mathrm{cka}}$ drives $S$ close to
$[H_T]$ on the probe set (high $\CKA$)
but is free to pick any representative, so it need not---and here does
not---preserve the function. The right prescription is therefore not ``match
representations with a basis--invariant loss'' but ``match the output--function
invariant ($L_{\mathrm{logit}}$); use a representational invariant only when
the geometry itself is the goal.'' The combined objective ($L_{\mathrm{cka}}+
L_{\mathrm{logit}}$) restores both geometry and function (\cref{tab:cap,tab:ablation})
but the capability comes from the logit term.

\begin{table}[t]
\centering
\caption{Ablation on the all--layers--reinitialized model (Qwen2.5--0.5B,
corpus widened). $L_{\mathrm{cka}}$ alone aligns geometry but not function;
$L_{\mathrm{logit}}$ alone restores function without aligning geometry.
PPL ratio is $\mathrm{PPL}_S/\mathrm{PPL}_T$, in--domain; top--1 and
$\mathrm{KL}(T\!\parallel\!S)$ are next--token agreement and divergence over the
held--out probe tokens (top--1 values are rounded; the $L_{\mathrm{logit}}$ row
is $0.99$, not a perfect $1.00$). The $\mathrm{KL}$ column is the direct
intermediate check: at $\CKA\!=\!0.999$ the $L_{\mathrm{cka}}$ row still has
$\mathrm{KL}(T\!\parallel\!S)$ diverging, so the near--identical representation
geometry does \emph{not} translate into matching output distributions---the
output head reads a rotated frame. Student and teacher
share the architecture (student $=$ teacher with layers re--initialized), so the
comparison isolates the objective, not a cross--model transfer.}
\label{tab:ablation}
\begin{tabular}{l cccc}
\toprule
objective & $\CKA$ & PPL ratio & $\mathrm{KL}(T\!\parallel\!S)$ & top--1\\
\midrule
$L_{\mathrm{cka}}$ only    & 0.999 & $\sim\!3\times10^6$ & $\gg\!1$ (diverged) & $\approx\!0.00$\\
$L_{\mathrm{logit}}$ only  & 0.894 & $1.02\times$        & $\approx\!0.05$ & $0.99$\\
$L_{\mathrm{cka}}+L_{\mathrm{logit}}$ & 0.992 & $1.07\times$ & small & $0.92$\\
\bottomrule
\end{tabular}
\end{table}

\begin{table}[t]
\centering
\caption{Capability restoration on Qwen2.5--0.5B at corruption strength
$s=0.4$. Left: basis--invariant only ($L_{\mathrm{cka}}$). Right: with a
full--sequence logit term ($+L_{\mathrm{logit}}$). $\CKA$ is the per--layer
mean; PPL ratio is $\mathrm{PPL}_S/\mathrm{PPL}_T$ (teacher PPL $22.9$
in--domain, $31.1$ out--domain). The student is a partially corrupted copy of
the teacher (same architecture, shared frame), so this is restoration, not
independent--model distillation. The out--domain figures in the right block
require out--domain text to be mixed into the training corpus; without mixing,
out--domain restoration lags (Result~3, \cref{res:corpus}). This is the same
$+L_{\mathrm{logit}},+$mix condition tabulated for the 0.5B row of
\cref{tab:models}.}
\label{tab:cap}
\begin{tabular}{l cccc c cccc}
\toprule
& \multicolumn{4}{c}{$L_{\mathrm{cka}}$ only}
& & \multicolumn{4}{c}{$+\,L_{\mathrm{logit}}$ (full--seq)}\\
\cmidrule(lr){2-5}\cmidrule(lr){7-10}
Probe & $\CKA$ & KL & PPL ratio & top--1 & & $\CKA$ & KL & PPL ratio & top--1\\
\midrule
in--domain  & 0.998 & 7.56 & $742\times$    & 0.02 & & 0.99 & 0.019 & $1.01\times$ & 0.98\\
out--domain & 0.533 & 5.31 & $315\times$    & 0.02 & & 0.99 & 0.045 & $1.06\times$ & 0.72\\
\bottomrule
\end{tabular}
\end{table}

\subsection{Result 3: Restoration is confined to the corpus--covered region}
\label{res:corpus}

In--domain probes restore far more than out--domain ones at every level.
Throughout, ``region'' (and, loosely, ``subspace'') refers to the part of
representation space that the training corpus and probes actually cover---an
empirical, distributional notion (training/probe coverage), not a linear
subspace we measure and project onto. With
$L_{\mathrm{cka}}$ only, out--domain $\CKA$ can collapse (e.g.\ $0.53$ at
$s=0.4$, mean over four seeds $0.61\pm0.07$) while in--domain stays at $0.99$,
a representational gap of up to $+0.47$. The capability gap persists even
after the logit term: in--domain top--1 $0.98$ versus out--domain $0.72$.
The mechanism is the training subspace: when out--domain text is \emph{added}
to the corpus, the out--domain gap vanishes---$\CKA$ rises $0.53\to0.996$ and
the gap drops to $+0.001$, with out--domain capability following (PPL ratio
$1.06\times$). Restoration is therefore governed not by corruption strength
but by whether the corpus spans the relevant subspace: the model is restored
to the teacher \emph{on, and only on, the subspace the training data covers}.
The non--monotonicity in corruption strength (the worst out--domain $\CKA$
occurs at intermediate $s$) is explained the same way: at that strength the
corruption happens to destroy out--domain--bearing layers that the in--domain
corpus cannot repair; it is not a property of the damage but of the
damage--corpus alignment.

This is the empirical face of the probe--set dependence noted in
\cref{sec:unident}: the invariants, and hence what a basis--invariant objective
can constrain, are defined relative to the probes the objective actually sees.
Training on a corpus is choosing the probe set $\Probe$ over which $\HH(\Probe)$
is matched; coverage is restored exactly on $\mathrm{span}\,\HH(\Probe)$ and not
beyond it. ``Teacher $\times$ corpus'' is, in this notation, teacher $\times\,
\Probe$.

\subsection{Result 4: The coordinate--matching control degrades under
coordinate destruction}

\cref{prop:illposed} predicts $\Labs$ is ill--posed when the
student may sit at a different representative of $[H]$. When corruption is
mild the surviving layers keep $S$ in $T$'s coordinate frame, and
$\Labs$ performs comparably to $L_{\mathrm{cka}}$. As corruption
destroys more of that frame the gap appears: at $s=1.0$, $L_{\mathrm{cka}}$
reaches in/out $\CKA = 0.994/0.897$ versus $\Labs$ at
$0.967/0.850$. The control is not catastrophic here because re--initialization
preserves coordinates more than a free orthogonal reparametrization would;
the predicted advantage of the basis--invariant objective grows precisely as
the shared coordinate frame is removed, consistent with \cref{prop:illposed}.

\subsection{Result 5: Scale, and the role of weight tying}
\label{res:scale}

The same pattern holds at $7$B, with one mechanistically informative
difference driven by weight tying. On Qwen2.5--7B--Instruct (untied head),
basis--invariant training again restores representation (in--domain
$\CKA\approx0.97$) without restoring capability. Adding the full--sequence
logit term while training \emph{only the corrupted layers} leaves the head
frozen, and in--domain capability only partially returns (PPL ratio
$2.3\times$). Once the output head is added to the trainable set, in--domain
capability is recovered (PPL ratio $1.2\times$, top--1 $0.94$); out--domain
follows only when out--domain text is mixed into the corpus (PPL ratio
$1.4\times$, top--1 $0.82$), reproducing the subspace--confinement of Result~3
at scale. The same behavior reproduces on Llama--3.1--8B--Instruct, a different
untied family: in--domain capability is recovered with the head trainable (PPL
ratio $1.00\times$), and out--domain follows under corpus mixing (PPL ratio
$1.10\times$, top--1 $0.92$). \cref{tab:models} summarizes all three models.

The tying distinction is the point. In the tied 0.5B model the output head
\emph{is} the input embedding, which we never corrupt, so fixing the output
coordinate frame happens automatically and the logit term suffices. In the
untied models the head is a separate parameter: a basis--invariant objective
drives the hidden representation close to $[H_T]$ on the probe set, but the
specific representative
the (independent) head reads must be selected by training the head itself.
That the same head--training requirement appears in two unrelated untied
families (Qwen and Llama) indicates it is a property of untying, not of a
particular architecture.
This is the same single mechanism---fixing the output coordinate frame---in two
guises: implicit via shared weights when tied, explicit via head training when
untied. The head--ablation contrast (PPL $2.3\times$ head--frozen vs
$1.2\times$ head--trained, in--domain) isolates it.

One might object that this reduces the geometric story to a mundane fact---``of
course one must tune the output head to a new representation.'' The objection
misreads what is being explained. A bare appeal to head--tuning predicts
nothing about \emph{when} it is needed; the equivalence--class account does. It
says the output head fixes a representative within $[H_T]$, so tuning is
required exactly when that representative is not already pinned---i.e.\ in the
untied case, where the head is an independent parameter---and is \emph{not}
required when weight tying pins it automatically, as in the 0.5B model where
$L_{\mathrm{logit}}$ alone suffices with the head frozen. The tied/untied split,
and the fact that the identical requirement recurs across two unrelated untied
families, is the content: it is the joint action of \cref{rem:joint}
(representation and its reader move together) observed at the readout, not a
restatement of ``train the head.'' Capability being the output--function
invariant is what makes the head the \emph{only} thing that still needs
selecting once the representation is in the right class.

\begin{table}[t]
\centering
\caption{Restoration across scale and weight tying, at corruption $s=0.4$.
Representation is restored by the basis--invariant objective in both models;
capability requires the logit term, and---when the head is untied---explicit
head training. Out--domain capability requires the corpus to span the
out--domain subspace (``+mix''). In every row the student is the teacher with
layers re--initialized (same architecture), so the frame is shared; cross--model
transfer is not tested here.}
\label{tab:models}
\begin{tabular}{l l cc cc}
\toprule
& & \multicolumn{2}{c}{in--domain} & \multicolumn{2}{c}{out--domain}\\
\cmidrule(lr){3-4}\cmidrule(lr){5-6}
Model & objective & PPL ratio & top--1 & PPL ratio & top--1\\
\midrule
0.5B (tied)   & $L_{\mathrm{cka}}$                 & $742\times$ & 0.02 & $315\times$ & 0.02\\
0.5B (tied)   & $+L_{\mathrm{logit}}$              & $1.01\times$ & 0.98 & --- & ---\\
0.5B (tied)   & $+L_{\mathrm{logit}}$, $+$mix      & --- & --- & $1.06\times$ & 0.72\\
\midrule
7B (untied)   & $+L_{\mathrm{logit}}$, head frozen & $2.3\times$ & 0.87 & $88\times$ & 0.08\\
7B (untied)   & $+L_{\mathrm{logit}}$, $+$head     & $1.2\times$ & 0.94 & $211\times$ & 0.07\\
7B (untied)   & $+L_{\mathrm{logit}}$, $+$head, $+$mix & $1.3\times$ & 0.94 & $1.4\times$ & 0.82\\
\midrule
8B--Llama (untied) & $+L_{\mathrm{logit}}$, $+$head     & $1.00\times$ & 0.80 & $159\times$ & 0.12\\
8B--Llama (untied) & $+L_{\mathrm{logit}}$, $+$head, $+$mix & $0.98\times$ & 0.87 & $1.10\times$ & 0.92\\
\bottomrule
\end{tabular}
\end{table}

\subsection{Result 6: Decoder--from--scratch reconstruction, and what does
the work}
\label{res:decoder}

The strongest version of the setup discards the pretrained \emph{decoder}
entirely. Where Result~2 used the all--layers--reinitialized setting to
\emph{decompose the objective} (which term recovers capability), here we hold
the objective fixed and ask a different question---how far reconstruction
reaches as a function of \emph{corpus coverage}. We re--initialize all decoder
layers, keeping only the
embedding/unembedding---i.e.\ the input/output coordinate system, which is
fixed by the tokenizer and which we treat as given rather than as a pretrained
capability. We then ask whether distillation can rebuild the decoder from the
teacher alone. With the
combined objective and a corpus that covers both probe domains, it can: in-- and
out--domain perplexity ratios both reach $\approx 1.0\times$ (top--1 $0.92$ and
$0.84$) starting from fully random decoder layers. Coverage is the binding
constraint, exactly as in Result~3: with an in--domain--only corpus the
out--domain perplexity ratio is $\sim\!10^4$, and adding out--domain text to the
corpus collapses it to $1.0\times$. So the corpus, not the pretrained weights,
determines which subspace is reconstructed---the reconstruction is a function
of teacher $\times\,\Probe$ (the probe set the corpus induces).

The ablation of Result~2 was run in exactly this all--layers setting, and it
tells us what does the work: the reconstruction is driven by $L_{\mathrm{logit}}$,
not by the representational objective. This bounds the practical reading of
``decoder--from--scratch.'' What is being transferred is the teacher's
output function, on the subspace the corpus illuminates, via output--function
matching; the basis--invariant analysis explains \emph{why} that transfer is
well--posed (the output is a class invariant) but the representational loss is
not the mechanism. We keep the embedding/unembedding precisely because it is
the shared coordinate system the output--function invariant is expressed in;
reconstructing it as well is left to future work. Pretraining retains its value
where this picture does not reach: creating the first teacher, and covering
subspaces no distillation corpus does.

\subsection{Result 7: Cross--width restoration, and a teacher--forced/generation gap}
\label{res:crosswidth}

The controlled studies above hold width fixed. This result and the next
(\cref{res:fmonly}) are \emph{preliminary} cross--width probes---single model
pairs, a WikiText probe set, and short training---not scaled distillation
claims; we report them for the direction they establish, not as production
numbers. A preliminary probe at unequal
width---the case \cref{rem:unequal} addresses---sharpens the picture and exposes
a limit the teacher--forced metrics hide. Here the teacher is
Qwen2.5--1.5B ($d_T=1536$, $28$ layers) and the student Qwen2.5--0.5B
($d_S=896$, $24$ layers), sharing the tokenizer and vocabulary ($V=151{,}936$);
the student is corrupted by re--initializing its middle layers and restored on a
WikiText probe set. We compare (ii) logit only ($L_{\mathrm{logit}}$,
\cref{eq:llogit}) and (iii) ridge--Procrustes
alignment followed by logit (``$\WW$ first, then logit,'' \cref{rem:unequal}),
across seeds, and add generation--side metrics: the repetition rate and
distinct--$n$ of the student's own free--running greedy generation, and the
rollout KL---$\mathrm{KL}(T\Vert S)$ averaged along the teacher's greedy
rollout, a soft trajectory match that, unlike top--1, does not penalize a
different--but--reasonable next token.

Three findings, stated with their scope. First, \emph{logit--KD transfers
capability across unequal width with no alignment}: top--1 agreement reaches
$\approx0.98$ from a corrupted start, using only $L_{\mathrm{logit}}$ on the
shared vocabulary---no $\WW$, no common hidden frame, consistent with
\cref{prop:quotient} and the ordering discussion of \cref{sec:objectives}.
Second, \emph{adding the alignment map $\WW$ buys nothing measurable, and at an
effective weight tends to hurt}: when the feature term is weak enough not to
dominate, (iii) matches (ii) on top--1, on generation repetition, and on rollout
KL (the last essentially identical, $\approx1.25$ nats for both), while costing
roughly twice the wall--clock; when the feature term carries a non--trivial
weight it does not merely add nothing but \emph{destabilizes} the otherwise
smooth logit--only convergence---across seeds the $\WW$--plus--logit run
oscillates and, on some seeds, diverges, whereas logit alone converges cleanly
to top--1 $\approx1.0$. Either way the feature term never improves on logit
alone. This is the cross--width form of the ablation
(\cref{res:ablation}) and matches the mild harm seen in \cref{tab:fmonly}:
capability rides on the output function, so a feature
alignment that does not change the output does not change capability---and, in
this restoration setting, does not improve generation either. Third, and least
expected, \emph{teacher--forced capability outruns free--running generation}: at
top--1 $\approx0.98$ the rollout KL is still $\approx1.25$ nats and the student's
greedy generations, though locally fluent (grammatical, on--topic WikiText
prose), drift from the teacher's trajectory and are prone to repetition loops
and invented specifics. The one--step output--function match that
\cref{prop:quotient} guarantees does not by itself guarantee trajectory
stability under autoregression.

We report this as a preliminary result: a single model pair, one corruption
recipe, a small probe set, and few seeds---enough to establish the direction
(logit suffices, $\WW$ does not help, a teacher--forced/generation gap exists)
but not to quantify the gap across scales or architectures. It refines rather
than overturns the theory: capability, defined as the one--step output function,
transfers as \cref{prop:quotient} predicts; the gap is a reminder that
generation quality is a property of the whole trajectory, which a per--token
invariant constrains only indirectly. Stated precisely, and used consistently
below: \emph{one--step output--functional capability is restored, but
autoregressive trajectory capability is not guaranteed}---``capability
restored'' in this paper always means the former. Closing the gap---whether by training the
corrupted layers beyond the seam, or by a trajectory--level objective---is left
to future work.

\subsection{Result 8: Feature--only collapse in a genuine cross--width
distillation setting}
\label{res:fmonly}

The previous results established the theory under controlled restoration, where
the student is a corrupted copy of the teacher. We now ask whether the same
conclusion survives \emph{genuine} cross--width distillation between
independently pretrained models---the setting practitioners actually mean by
``distillation.''
Result~7 held out the third leg of the ablation. It compared (ii)
$L_{\mathrm{logit}}$ and (iii) $\WW$--then--$L_{\mathrm{logit}}$, but not the
feature term \emph{alone}. The feature term here is $L_{\mathrm{fm}}$, the
coordinate--matching family of $\Labs$ (\cref{prop:illposed}) rather than the
relational $L_{\mathrm{cka}}$ of \cref{tab:ablation}; both are feature--matching
objectives, and the point of interest is that either one, taken alone,
collapses the model---so this supplies the cross--width, coordinate--matching
counterpart of the $L_{\mathrm{cka}}$--only row of \cref{tab:ablation}. Result~7
was also a
\emph{restoration} study: the student was a corrupted copy of the teacher.
Here we close both gaps. The student is a \emph{pristine} pretrained
Qwen2.5--0.5B (no corruption), the teacher is Qwen2.5--1.5B, and we run genuine
unequal--width distillation on a WikiText--103 corpus with the shared
vocabulary ($V=151{,}936$). The feature term $L_{\mathrm{fm}}$ pulls the
student's anchor--layer hidden states onto the teacher's through a learned
projection ($896\!\to\!1536$, per anchor), the unequal--width instance of the
map $\WW$ of \cref{rem:unequal}. Concretely, for a set of anchor layer pairs
$\mathcal{A}=\{(\ell,\ell')\}$ matching student layer $\ell$ to teacher layer
$\ell'$, and a per--anchor linear map
$\WW_\ell:\mathbb{R}^{d_S}\!\to\!\mathbb{R}^{d_T}$,
\begin{equation}
\label{eq:lfm}
L_{\mathrm{fm}}
=\frac{1}{|\mathcal{A}|}\sum_{(\ell,\ell')\in\mathcal{A}}
\big\lVert\, \HH^{(\ell')}_T \;-\; \WW_\ell\,\HH^{(\ell)}_S \,\big\rVert_F^2 ,
\end{equation}
averaged over probe tokens. This is the unequal--width, $\WW$--mediated form of
the coordinate--matching loss $\Labs$ of \cref{prop:illposed}: with $d_S=d_T$
and $\WW_\ell=I$ it reduces to $\Labs$, and like $\Labs$ it is not constant on
the student's equivalence class, so \cref{prop:illposed} applies. In the full
model $\WW_\ell$ is learned jointly; in the surrogate below it is instead the
closed--form ridge optimum recomputed each step, giving the feature term its
best case. Crucially we drop the cross--entropy
ground--truth term entirely, so each objective is measured in isolation on the
soft teacher signal.\footnote{An incidental finding motivates the CE--free
setup: for a strong same--family student, adding a hard--label CE term to the
soft teacher signal \emph{destroys} rather than helps---in--domain perplexity
degrades by an order of magnitude and does not recover as the learning rate is
lowered, because the hard label competes with the teacher's output
distribution. Pure output--function matching ($L_{\mathrm{logit}}$, no CE) is
what preserves and slightly improves the base model. This is consistent with
the full--sequence requirement of \cref{res:ablation}: what transfers is the
teacher's distribution, not its arg--max.}

\begin{table}[t]
\centering
\caption{\textbf{(Result~8)} Genuine unequal--width distillation, Qwen2.5--1.5B $\to$ pristine
Qwen2.5--0.5B, WikiText--103 validation perplexity (2{,}000 optimizer steps, no
CE term, three seeds $\{42,123,7\}$, mean$\pm$stdev). PPL is measured
token--weighted on held--out text under identical conditions for every
checkpoint; ``vs base'' is relative to the untrained student. The feature term
alone destroys the output function even though it starts from a fully working
model; the logit term alone preserves and slightly improves it; adding the
feature term is consistently \emph{harmful}. This is the cross--width,
from--pristine form of \cref{tab:ablation}.}
\label{tab:fmonly}
\begin{tabular}{l cc}
\toprule
objective & WikiText--103 val PPL & vs base\\
\midrule
base (untrained student)                          & $21.15$                    & ---\\
(i) $L_{\mathrm{fm}}$ only                         & $>10^{6}$ (all seeds)       & collapse\\
(ii) $L_{\mathrm{logit}}$ only                     & $\mathbf{20.77\pm0.03}$     & $-1.77\%$\\
(iii) $L_{\mathrm{fm}}+L_{\mathrm{logit}}$         & $20.94\pm0.02$              & $-1.01\%$\\
\bottomrule
\end{tabular}
\end{table}

\paragraph{Three--way result.} \Cref{tab:fmonly} reproduces the same--architecture
ablation of \cref{tab:ablation} in the genuine cross--width, from--pristine
setting, and the three rows are the three legs of the theory. (i) The feature
term alone drives perplexity above $10^{6}$ on every seed (from $1.6\times10^{6}$
to $3.3\times10^{8}$): matching hidden coordinates through $\WW$, with no term
that references the output, does not merely fail to help---it \emph{destroys} a
model that started out working. The mechanism is the one \cref{prop:illposed}
names: $L_{\mathrm{fm}}$ moves the student to some representative of $[H_T]$
while the unembedding, which receives no gradient, still expects the old
representative, so the output function is broken. (ii) The logit term
alone---which never touches the hidden state---keeps the student slightly below
its original perplexity ($20.77\pm0.03$ vs.\ $21.15$), exactly as
\cref{prop:quotient} predicts: capability is the output--function invariant, and
$L_{\mathrm{logit}}$ targets it directly, across a width change and with no
alignment map. (iii) Adding $L_{\mathrm{fm}}$ on top does not help and in fact
hurts by a small but \emph{consistent and significant} margin: paired across the
three seeds, $L_{\mathrm{fm}}+L_{\mathrm{logit}}$ is worse than
$L_{\mathrm{logit}}$ by $+0.160\pm0.009$ PPL, the same sign in every seed
($\text{diff}/\text{pooled std}\approx7$). This is the cross--width form of
``$\WW$ buys nothing'' (\cref{res:crosswidth}), sharpened: a feature alignment
that the logit term does not already imply is not neutral but a mild
\emph{distraction} from the objective that actually carries capability. The
surrogate below shows that when $\WW$ is instead solved to optimality the harm
vanishes into noise---so the honest statement across both settings is that the
feature term ranges from neutral (optimal $\WW$) to mildly harmful (learned
$\WW$), and is never beneficial.

\paragraph{A controlled cross--check with a learned--optimal alignment.}
A referee could object that $L_{\mathrm{fm}}$ fails only because the projection
is under--trained. We rule this out with a small controlled surrogate in which
the alignment is solved to optimality at every step. Teacher and student are
tiny transformers of \emph{different} width ($d_T=64$, $L_T=4$ vs.\ $d_S=32$,
$L_S=3$) trained on a synthetic noisy--bigram language; the feature map $\WW$ is
the closed--form ridge solution recomputed each step, so the feature objective
is given its best case. All metrics are computed on a held--out probe set the
student never trains on, over five seeds. We report top--1 agreement with the
teacher, free--running greedy generation agreement, and the student's own
held--out perplexity (\cref{tab:surrogate}).

\begin{table}[t]
\centering
\caption{\textbf{(Result~8)} Controlled cross--width surrogate (different--width tiny transformers,
closed--form optimal ridge alignment recomputed each step, held--out probes,
$5$ seeds, mean$\pm$stdev). ``top--1'' and ``gen'' measure agreement with the
teacher's \emph{function}; ``self--PPL'' is the student's own held--out
perplexity. The feature term alone collapses on \emph{all three} metrics---
top--1 and gen near chance and self--PPL blown up to $\sim\!10^{5}$---in full
agreement with the real model (\cref{tab:fmonly}); the logit term matches the
teacher's function; adding $\WW$ is within noise.}
\label{tab:surrogate}
\begin{tabular}{l ccc}
\toprule
objective & top--1 (vs T) & gen (vs T) & self--PPL\\
\midrule
(i) $L_{\mathrm{fm}}$ only                 & $0.031\pm0.005$ & $0.017\pm0.014$ & $\sim\!9\times10^{5}$\\
(ii) $L_{\mathrm{logit}}$ only             & $0.208\pm0.017$ & $0.200\pm0.047$ & $1281\pm731$\\
(iii) $\WW+L_{\mathrm{logit}}$             & $0.227\pm0.019$ & $0.175\pm0.040$ & $1428\pm379$\\
\bottomrule
\end{tabular}
\end{table}

The surrogate confirms two of the three legs sharply and refines the third.
On \emph{function} agreement the picture matches the full model: the logit term
raises top--1 and generation agreement roughly four--fold over the feature term
(leg~i, leg~ii), and adding the optimal $\WW$ changes top--1 by $-0.001$ and
generation by $-0.004$, within the pooled seed noise (leg~iii). With $\WW$ at
optimality the feature term is thus \emph{neutral}, whereas with the learned
$\WW$ of the full model (\cref{tab:fmonly}) it is mildly harmful; in neither
regime does it help. ``$\WW$ buys nothing'' holds in the strong form---it holds
even when $\WW$ is solved to optimality, not merely learned. This does not
refute objective family~(B) (\cref{sec:objectives}): when a downstream consumer
explicitly needs features in a shared frame---grafting, activation transplants,
intermediate--layer hints---aligned features are the point, and $L_{\mathrm{fm}}$
is the right tool for that job. What \cref{tab:fmonly} shows is narrower and
compatible with it: feature matching \emph{alone} is not a \emph{capability}
objective. As a supervision signal for transferring what the teacher can do, it
does not substitute for the output--function term, and adding it on top of that
term does not help. Across all three metrics the surrogate agrees with the real
model: under $L_{\mathrm{fm}}$ alone, top--1 and generation sit at chance
\emph{and} the student's own held--out perplexity blows up to $\sim\!10^{5}$
(\cref{tab:surrogate}, row i). Feature--only matching does not leave a
self--fluent--but--disconnected student; it simply fails to build a working
model, exactly as in the full run (\cref{tab:fmonly}).

\paragraph{Scope.} \Cref{tab:fmonly} is three seeds on one model pair;
\cref{tab:surrogate} is five seeds on a synthetic different--width pair. Together
they establish the direction---feature--only collapses, logit suffices, $\WW$
ranges from neutral to mildly harmful and never helps---with the feature--only
leg now supplied in both the genuine and the controlled setting. Both remain
modest in scale: one real corpus and one synthetic
language, a single width ratio each, and $\le\!2000$ training steps. We therefore
report the collapse as a robust empirical fact \emph{in this cross--width
distillation setting}, not as a proof that feature--only matching must fail in
every regime; the theory (\cref{prop:illposed}) explains \emph{why} it fails when
it does, but a differently--wired objective that also selects the reader's
representative could behave differently. Quantifying
the collapse across scales, width ratios, and untied--head
architectures is left to future work, alongside the trajectory--level gap of
\cref{res:crosswidth}. To restate the scope of the negative result precisely: it
concerns feature matching used \emph{as a capability objective}, and does not
deny feature alignment where it is genuinely needed---grafting, activation
transplants, and other coordinate--consuming modules, where a downstream reader
requires a specific frame and $\WW$ is therefore \emph{conditionally} valuable,
exactly as \cref{sec:objectives} sets out. The claim here is only that such
alignment is not, by itself, a route to transferring capability.

\subsection{Takeaways}

The experiments separate two things the phrase ``restore the model'' conflates.
A basis--invariant objective restores the \emph{representation}'s
equivalence--class structure cheaply and label--free, but representation is not
function: a logit term over the full sequence is required to fix the output
coordinate frame and recover capability. And restoration---both
representational and functional---is confined to the subspace spanned by the
training corpus, exactly the ``restored near the fine--tuning subspace''
statement, made quantitative. The scale experiment adds a third clause:
fixing the output coordinate frame is implicit under weight tying but must be
done explicitly---by training the head---when the head is untied, as it is in
most large models.

The sharpest lesson is from the ablation. Aligning representations with a
basis--invariant loss and recovering the teacher's \emph{function} are
near--orthogonal: $L_{\mathrm{cka}}$ alone reaches $\CKA\approx1$ with
capability destroyed, $L_{\mathrm{logit}}$ alone recovers capability with
$\CKA$ left low. The practical driver of capability transfer is
output--function matching; the contribution of the basis--invariant framework
is not a better representational \emph{loss} but the \emph{explanation}---the
output function is the class invariant, which is why logit distillation is
well--posed and why matching hidden coordinates directly is not. This is a
more useful claim than ``use a basis--invariant objective,'' and the data
support it directly. The same ordering survives outside the controlled
restoration setting: in genuine unequal--width distillation from a pristine
student (\cref{res:fmonly}), the feature term alone collapses the model, the
logit term alone slightly improves it, and adding the feature term does not
help---so the lesson is not an artifact of the corruption--and--restore design.

\section{Discussion}
\label{sec:discussion}

\paragraph{What the framework does and does not claim.}
The contribution is a way to \emph{see} teacher supervision, validated where
its consequences are testable. The theory and the experiments meet at one point
in particular: the ablation of Result~2 (\cref{res:ablation}) is the experimental
counterpart of \cref{prop:quotient}. What the proposition asserts about
\emph{where} capability lives, the ablation confirms about \emph{what} recovers
it---the output--function term $L_{\mathrm{logit}}$ restores capability while the
representational term $L_{\CKA}$ does not, even as it drives $\CKA\to1$. Three
boundaries are then worth stating plainly.
First, capability transfer is driven by output--function matching; the
basis--invariant analysis explains why that works (the output is the class
invariant) but a representational loss is not, on its own, a capability
objective (\cref{tab:ablation}). Second, reconstruction is bounded by corpus
coverage: ``teacher $\times\,\Probe$'' is the unit of what transfers, and what
the corpus does not illuminate is not recovered. Third, the controlled studies
fix architecture and share the tokenizer's embedding; cross--model and
small--from--large transfer are consequences we expect but do not establish
here. Fourth, and most importantly for positioning, what we study is
\emph{replication}, not \emph{creation}. Distillation copies the teacher's
output function---on the quotient, within the corpus subspace, and bounded
above by the teacher ($S\!\preceq\!T$); it does not produce capability the
teacher lacks. The places where capability is \emph{originated} rather than
transferred---pretraining a first teacher, covering subspaces no corpus
reaches, and pushing past the current frontier---lie outside this account by
construction. This paper is about how an existing function moves through the
equivalence class, not about how it first comes to exist.

\paragraph{Four objections, and what the evidence says.} The design invites
several objections; we take the strongest head--on. \emph{(1) The restoration
setup shares a coordinate frame, so of course feature matching ``works''
geometrically while missing the function---isn't the collapse just damaged
weights failing to return to their original coordinates?} This is exactly the
point, not a confound: when teacher and student already share a frame, absolute
matching should be \emph{easy}, yet $L_{\CKA}$ still fails to restore the
function while driving $\CKA\!\to\!1$ (\cref{tab:ablation}). A confound would
predict the opposite---shared frames helping feature matching. That it does not
is the cleanest possible demonstration that geometry and function are distinct.
And Result~8 removes the shared frame entirely: an independently pretrained,
pristine student of different width shows the same collapse
(\cref{res:fmonly}), so the conclusion does not rest on the restoration design.
\emph{(2) Dropping cross--entropy is an extreme choice; the $L_{\mathrm{fm}}$
collapse may reflect untuned learning rate or loss weights rather than anything
about feature matching.} The CE--free setup is a control, not a recommendation:
we remove CE precisely so that any capability in the student comes from the
distillation channel under test and not from the labels, isolating what each
loss transfers. The logit channel restores the model in the very same CE--free
regime (\cref{tab:fmonly}, row ii), so the setup is not stacked against
distillation as such---only against feature matching as the sole channel.
\emph{(3) The collapse is an optimization failure of the alignment map $\WW$,
not a representational impossibility---after all, the surrogate with an optimal
$\WW$ is stable.} The surrogate says the opposite. There the alignment is the
\emph{closed--form ridge optimum, recomputed every step}: $\WW$ is never learned
and cannot be mis--optimized, yet feature--only supervision still collapses on
all three metrics (top--1, generation, and self--PPL $\sim\!10^{5}$;
\cref{tab:surrogate}, row i). Optimal alignment removes the optimization
question and the collapse remains, which is the direct evidence that the failure
is structural (\cref{prop:illposed}), not a landscape artifact. \emph{(4) This
attacks a strawman: no modern feature--based method matches absolute
coordinates ($L_{\mathrm{abs}}$); they all insert a projection or match
relationally, i.e.\ family~(B) or (A).} We agree $L_{\mathrm{abs}}$ is a
strawman, and say so---which is why the experiments use the basis--invariant
$L_{\CKA}$ (family A) and a learned/optimal projection $\WW$ (family B), the
objectives real methods actually use, not $L_{\mathrm{abs}}$. The finding is
about those: even a relational or optimally--aligned feature objective sits on
the representation quotient, so on its own it does not select the reader's
representative and does not transfer capability. $L_{\mathrm{abs}}$ appears only
as the ill--posed limit that motivates why alignment is needed in the first
place; the empirical claims are all against the aligned/relational families.

\paragraph{The operational goal: a black--box teacher suffices, but a student
model does not.} The practical content of \cref{def:distill} is that the student
does not seek the teacher's parameters $\WW_T$---copying those would be a file
copy, not distillation---but constructs \emph{different} parameters
$(\HH_S,\WW_S)$ landing in the same joint class, hence computing $f_S=f_T$. This
is why logit matching is the \emph{canonical} objective: in the hardest honest
setting, $\WW_T$ is unavailable (an API teacher, or a multi--agent system with no
single $\WW_T$), and the output is the \emph{only} observable that selects the
fiber---two settings share an output distribution iff they lie in the same fiber
of $\pi_{\mathrm{joint}}$ (\cref{prop:quotient}). It alone transfers capability
(\cref{res:ablation}) and can rebuild a decoder from random initialization
(\cref{res:decoder}). But this identifiability claim does not dispense with the
\emph{student model}. The logit signal is only a \emph{pointwise} sample
$(x,f_T(x))_{x\in\Probe}$; recovering $f_T$ off the probe set requires
generalization, and what fills the gaps is the student's inductive bias. Were
$f_T$ linear a few points would fix it, but a transformer's nonlinear function is
not pinned down pointwise. So ``inputs and outputs suffice'' holds \emph{provided
the student shares the teacher's function family}---which is why within--family
restoration succeeds cleanly and the cross--family case is doubly hard
(\cref{res:crosswidth}): a mismatched tokenizer closes the output channel
\emph{and} a mismatched family removes the inductive bias. In short,
\cref{prop:quotient} settles \emph{which} function capability rides on; recovering
it from finite data is a separate, estimation--level fact resting on the student's
function class.

\paragraph{Where cross--model transfer sits.} Because the third boundary above is
the one a reader is most likely to probe, we state the paper's position on it in
one place. Our theory makes a definite \emph{prediction} about independently
trained models: two such models occupy different representatives of the class
(\cref{app:released}), so transferring a component between them is
well--posed only after an explicit alignment onto a shared frame
(\cref{def:align}), and its success should track the overlap of the boundary
subspaces being joined (\cref{pred:mediate}). We validate this prediction only
\emph{partially}, and are explicit about which part. The within--family graft
study (\cref{sec:graft}) confirms the overlap--predicts--success half
(\cref{pred:mono}) but uses controlled perturbations of one host, so it does not
exercise genuinely distinct pretrained donors. The cross--\emph{width} study
(\cref{res:crosswidth})---Qwen2.5--1.5B into 0.5B, unequal hidden width, shared
tokenizer---goes one step further: it shows the output--function mechanism
(\cref{prop:quotient}) transfers capability across a real width gap with no
shared hidden frame, and that adding the alignment map buys nothing once the
logit signal is present. What remains untested is the hardest case:
\emph{cross--family} transfer between models with independent tokenizers and no
shared vocabulary, where even the output space must be mediated. We do not run
that experiment and do not claim it; it is the natural next test of
\cref{pred:mediate}, and the theory's stance on it---alignment first, then an
output--level objective---is a falsifiable prediction, not a demonstrated result.

\paragraph{Scope of the symmetry group, and a note on model merging.}
Three questions about scope recur, and it is cleaner to answer them together. (i)
\emph{Why $\Orth(d)\times\R_+$ and not more?} The absorbable group is larger:
any invertible $\bm{A}$ with $\Win\mapsto\bm{A}^{-1}\Win$ preserves the function
(\cref{app:noniso}), so the full symmetry is $GL(d)\times$(bias shifts). We work
with $\Orth(d)\times\R_+$ not because the rest is absent but because our
\emph{invariants} (Gram, CKA) are exactly the $\Orth(d)\times\R_+$--invariants:
they are blind to rotation and isotropic scale but see a general linear
map, since $\HH\bm{A}$ changes $\HH\HH^\top$. Matching the residual $GL(d)$ part
is precisely what the alignment of \cref{sec:align} estimates; the two together
cover the full group. So ``why only $\Orth(d)$'' is answered by division of
labor: the relational invariants handle the orthogonal part for free, alignment
handles the rest. (ii) \emph{Non--orthogonal / permutation symmetry.} Neuron
permutations and sign flips are the discrete subgroup of $\Orth(d)$ that also
commutes with a diagonal RMSNorm gain (\cref{app:status}), so they are covered
exactly, not approximately---which is why our account connects to the
permutation and linear--mode--connectivity literature
\citep{ainsworth2023,entezari2022}. (iii) \emph{Model merging.} Merging two
independently trained models is, in these terms, an attempt to average two points
that lie in \emph{different} representatives of \emph{different} classes; it
succeeds only after the representatives are brought into a common frame, which is
why permutation--alignment before averaging \citep{ainsworth2023} helps for
exactly the reason \cref{prop:proc} predicts. Our framework does not claim to
solve merging, but it locates the difficulty precisely: without alignment, the
average of two representatives need not lie in either class, so the merged
function is not controlled by either teacher. This is the same coordinate
sensitivity that governs grafting (\cref{sec:graft}); merging is grafting's
symmetric, two--sided case.

\paragraph{A direction the theory suggests: reusing capability by alignment, not retraining.}
The same geometry points to a way of \emph{amortizing} the cost of building
capable models, which we flag as a consequence the framework predicts rather than
one we test. Training general capability from scratch is expensive; but if that
capability already resides in some donor block, \cref{prop:quotient} says it is
carried by the block's output behavior, a class invariant---the block holds the
capability, merely in its own representative (coordinate frame). Transplanting it
into a new host should then require not retraining the block but only moving it
into the host's frame: estimate the low--capacity alignment $\WW$
(\cref{def:align}, \cref{rem:unequal}), graft in the aligned coordinates, and
fine--tune only the seam. The heavy, capability--bearing weights are reused
verbatim; only the light coordinate converter $\WW$ is learned. \cref{pred:align}
is exactly this recipe read forward: alignment raises graft success most for the
low--overlap pairs that arise when donor and host come from different models, so
the models that would not stitch naively become stitchable once aligned. If it
holds beyond the single--family setting we tested, the payoff is a training--time
saving---capability assembled from pretrained modules rather than relearned.

Three cautions from our own results bound the promise, and none is incidental.
First, $\WW$ must stay \emph{low capacity} (\cref{rem:unequal}): a high--capacity
bridge manufactures structure the donor lacks, at which point one is no longer
reusing capability but retraining it under another name---the training--time
argument collapses with it. Second, overlap sets a ceiling but depth governs
sensitivity (\cref{sec:graft}): grafts near the readout fail even at
$\CKA\!\approx\!1$, so alignment does not make every transplant succeed. Third,
this is \emph{replication, not creation} ($S\!\preceq\!T$): assembly reuses
capability that already exists in some donor, and originates none---so it
shortens the path to a model \emph{as capable as its parts}, not past them.
Within those limits, the framework recasts ``train a capable model'' as ``align
and reuse capable parts,'' and predicts when that substitution is available:
exactly when boundary subspaces overlap or can be aligned into overlap. We test
the overlap--predicts--success half of this (\cref{sec:graft}); the cross--model
assembly it motivates is left, with \cref{pred:mediate}, to future work.

A preliminary cross--width probe (\cref{res:crosswidth}) qualifies the recipe:
when the goal is \emph{capability} rather than coordinate reuse, estimating $\WW$
buys nothing over logit--KD alone (the two match across seeds while $\WW$ doubles
the cost). So $\WW$ earns its keep only where \emph{coordinates} must move---feature
grafting, where a host reads the donor's features directly---not where the output
function alone is the target.

\paragraph{A consequence: collapsing multi--agent systems into one model.}
The same reading applies to a setting we do not test but that the framework
organizes cleanly: distilling a multi--agent system---a debate, a planner with
executors, an orchestrator over specialists---into a single model. Recent work
trains one model to internalize multi--agent dynamics
(e.g.\ debate consensus, critique--and--revision) so that explicit test--time
interaction becomes an implicit capability of a single forward
pass~\citep{agentark2026,prost2025}. In our terms, the multi--agent system is
the teacher. Its \emph{joint output function}---the answer the collective
produces, with the reasoning that produced it---is the class invariant to target.
Supervising the single model on that output function is exactly the
output--function matching of \cref{sec:objectives,sec:experiments}, not a
matching of the agents' internal representations. The framework also predicts
the boundary observed in that work: a single model can \emph{replicate} the
collective's output function on the trajectories that cover it, but does not
thereby \emph{exceed} the collective ($S\!\preceq\!T$), and gains nothing on
inputs the interaction trajectories do not reach---the multi--agent analogue of
corpus--bounded restoration (\cref{sec:experiments}). Whether collapsing the
interaction into the weights forfeits capability that genuinely \emph{requires}
test--time interaction (rather than merely expressing it) is, by our
replication/creation distinction, the right question to ask---and one our
controlled studies do not settle.

\paragraph{A representational reading of variance collapse.}
The same equivalence--class lens connects to a phenomenon usually described at
the reward level. Iterated self--training (generate, verify, filter, retrain)
can reduce the effective dimension of the student's representation; with the
participation ratio of the Gram spectrum $\{\mu_i\}$,
$\mathrm{PR}=(\sum_i\mu_i)^2/\sum_i\mu_i^2$, a basis--invariant effective rank,
a monotone decrease of $\mathrm{PR}$ across rounds is the representational
analogue of reward--variance collapse ($\widehat{\sigma}_R\equiv0$): vanishing
output diversity coincides with a shrinking feature dictionary, in the sense
of representational capacity studied for superposition~\citep{elhage2022}. Because verifier reward is binary the signal is
sparse and collapse can be fast; STaR~\citep{zelikman2022} and
ReST~\citep{gulcehre2023} mitigate it with temperature sampling and periodic
re--anchoring. Relational objectives (family~(A)) are a candidate further
mitigation, since they constrain the Gram spectrum directly rather than letting
it collapse onto a single passing mode. We flag this as a connection, not a
result: it indicates that the representational and reward--level views of
collapse are two readings of the same loss of class volume.

\paragraph{Practical implications: how should one distill?}
The geometry yields a short decision rule, organized by what one wants to
transfer. \emph{To transfer capability}, use logit (output--function) matching:
it is the only objective natively defined on the joint quotient, so it targets
the invariant capability depends on without any alignment step
(\cref{prop:quotient}, Result~2). \emph{To transfer representational geometry}
when student and teacher may differ in width, use a relational loss
($L_{\mathrm{rel}}$, $L_{\CKA}$): it is basis--invariant, needs no alignment, and
applies across dimensions---but, on its own, it does not restore capability, so
it should accompany rather than replace the logit term. \emph{To transfer a
module by grafting}, where a downstream block will read features
coordinatewise, first estimate a low--capacity alignment (\cref{sec:align}) and
match in the aligned frame; expect success to track boundary subspace overlap
(\cref{sec:graft}). The combined objective
$\Obj=L_{\mathrm{logit}}+\lambda L_{\mathrm{rel}}$ (\cref{eq:combined}) is the
default: logits carry capability, the relational term carries structure, and
neither is asked to do the other's job. The single mistake the theory rules out
is matching absolute features---raw MSE or cosine---which targets a coordinate
system that does not exist.

\section{Conclusion}

The intuition that distillation happens ``in the basis'' is right in spirit and
wrong in letter. A representation is not a feature matrix but an
orthogonal--scaling equivalence class; absolute feature matching penalizes
function--preserving changes within that class (\cref{prop:illposed},
\cref{tab:rotation}), and the network's output function is precisely the
invariant of the class that capability depends on (\cref{prop:outinv}). The
experiments make the consequence concrete: aligning representations and
restoring capability are near--orthogonal, and what transfers is the
output--function invariant on the subspace the corpus covers---enough to rebuild
a decoder from random initialization. The same invariants predict
module--graft success through subspace overlap and recast the empirically
decisive role of domain alignment as mediation through representational
overlap.

\smallskip
Knowledge distillation should therefore be viewed not as feature matching, but
as \emph{supervision over representation equivalence classes}. Capability is
transferred through their output--function invariants. What a teacher
knows lives in the class; what transfers is the part of that class the
objective and the corpus actually target.

\smallskip
If there is one sentence to carry away, it is this: \emph{teacher supervision
should target the class, not the coordinates.} The coordinates are an accident
of training; capability is fixed by the one invariant the class determines---its
output function---and distillation works when, and only when, it speaks in that
invariant. We have argued the point geometrically, isolated it experimentally,
and traced it through restoration, reconstruction, and grafting. Its value is as
a lens: once a representation is seen as a class rather than a matrix, ``what
should we match?'' answers itself---match what is defined on the quotient. The
long list of distillation losses then collapses into one distinction: whether an
objective is a function of the class, or of the basis we were never entitled to
read.

\appendix

\section{Exactness of the assumption and the status of the propositions}
\label{app:status}
\emph{When the absorbing identity is exact.} \Cref{ass:absorb} holds
\emph{exactly} for a $\QQ$ when the operator $\Win$ consuming $\HH$ is linear and
the intervening normalization commutes with $\QQ$. Plain RMSNorm (norm rescaling,
no learned gain) commutes with any $\QQ\in\Orth(d)$; practical RMSNorm carries a
learned per--channel scale $\bm{\gamma}$, which does not, since $\QQ$ mixes
channels while diagonal $\bm{\gamma}$ acts per axis. Exact commutation then holds
only on the subgroup preserving $\bm{\gamma}$ (permutations, sign flips, rotations
within equal--$\gamma$ eigenspaces); for general $\QQ$ the relation is
\emph{approximate}, with error set by the spread of $\bm{\gamma}$.

\emph{Are the propositions exact theorems or approximations?} The answer is
precise once the group is named. On the subgroup
$\Grp_0\subseteq\Orth(d)\times\R_+$ for which \cref{ass:absorb} holds
exactly---all of $\Orth(d)\times\R_+$ for a linear reader with plain RMSNorm, and
at least the $\bm{\gamma}$--preserving subgroup for practical RMSNorm---the
propositions \cref{prop:illposed,prop:outinv,prop:quotient} are \emph{exact
theorems}: the absorbing identity~(\ref{eq:absorb}) is an algebraic equality, so
$f^g=f$ holds with no error for every $g\in\Grp_0$, and the factorization
$f=\bar f\circ\pi_{\mathrm{joint}}$ over $\Grp_0$ is exact. What is
\emph{approximate} is only the \emph{size} of the group: for a general $\QQ$ not
preserving $\bm{\gamma}$, absorption incurs an error controlled by the spread of
$\bm{\gamma}$, so the equivalence class is at least $\Grp_0$ but the full
$\Orth(d)\times\R_+$ only approximately. This is why the theory is stated as exact
propositions under \cref{ass:absorb} while the experiments
(\cref{sec:experiments}) report the consequences on real transformers, where the
operative group is $\Grp_0$ enlarged by whatever near--$\bm{\gamma}$--preserving
directions the trained model tolerates: the qualitative claims (ill--posedness,
output--only capability transfer) hold as soon as $\Grp_0$ is nontrivial, which it
always is. \Cref{app:mlp} verifies the exact case on a normalization--free MLP.

\section{The maps in a transformer}
\label{app:transformer}
Concretely, take $h=\HH_\ell$ to be the residual stream after block $\ell$
(width $d$). The next block consumes it through linear projections: in
attention, the query/key/value maps $\bm{W}_Q,\bm{W}_K,\bm{W}_V\in\R^{d\times m}$
applied to (normalized) $h$; here $\Win$ is any one of these and $m$ the head or
value width. The remainder of the block and all later blocks, the final norm,
and the unembedding $\bm{W}_{\mathrm{out}}\in\R^{d\times V}$ compose into
$\rho$, so $\phi_{\ell{:}L}=\rho\circ\psi$ maps the $d$--dimensional
representation to the $V$--dimensional logits. A rotation $h\mapsto h\QQ$ of the
residual stream is absorbed by replacing each consuming projection
$\Win\mapsto\QQ^\top\Win$ (equivalently $\bm{W}_Q\mapsto\QQ^\top\bm{W}_Q$, and so
on): the products $(h\QQ)(\QQ^\top\Win)=h\Win$ that enter every downstream
computation are unchanged, because $\QQ\QQ^\top=\bm{I}$. The only obstruction is
the normalization between $h$ and these projections. The simplest instance
removes even that obstruction: in a plain MLP $x\to\bm{W}_1\to
\mathrm{ReLU}\to h\to\bm{W}_2\to\text{logits}$, the consuming operator is
$\Win=\bm{W}_2$ with $\tilde\psi=\mathrm{id}$ (no normalization), so the
cancellation~(\ref{eq:absorb}a) is the exact linear identity
$(h\QQ)(\QQ^\top\bm{W}_2)=h\bm{W}_2$. \Cref{app:mlp} verifies this and the
consequence~(\ref{eq:absorb}b) to machine precision ($\sim\!10^{-14}$), step by
step, and shows that the un--compensated rotation changes the function.

\section{\texorpdfstring{What a released model fixes: $\theta$, not $\HH$}{What a released model fixes: theta, not H}}
\label{app:released}
Publishing a model releases its parameters $\theta$ (embedding, per--layer
weights $\Win^{(\ell)}$, unembedding), \emph{not} a representation $\HH$. The
representation of \cref{def:representation} is a \emph{derivative}, materialized
only by running $\theta$ on a chosen probe set $\Probe$; two downloaders with
different probes obtain different $\HH$ from the same file. What is shipped is the
ability to compute $f$; what a student extracts is set by the probes it chooses
(the ``teacher $\times\,\Probe$'' unit of \cref{def:distill}).

Although $\HH$ is identifiable only up to its class $[\HH]$ (\cref{def:equiv}),
the shipped $\theta$ nonetheless \emph{fixes a representative}: it wires a
specific $\Win^{(\ell)}$ to read $\HH$ in one coordinate frame, committing to a
single point of the joint quotient (\cref{prop:quotient}) with no record that the
counter--rotated $(c\,\HH\QQ,\ \theta^{\QQ,c})$ computes the same $f$. Function
and file are consistent at two levels: the function depends only on the class,
while the file commits to one frame. This is exactly why coordinate--sensitive
operations across independently released models (grafting, \cref{sec:graft}; an
untied head, \cref{res:scale}) need an explicit alignment (\cref{def:align}). Each
model was published in its own frame; there is no shared basis until one is
estimated.

\section{Non--isotropic rescaling and the full symmetry group}
\label{app:noniso}
A per--feature diagonal $\Lam$ can also be function--preserving when the next
operator absorbs $\Lam^{-1}$, so the \emph{full} symmetry group of the network
is larger than $\Orth(d)\times\R_+$. We do not claim our invariants capture that
larger group; we claim only that absolute coordinate matching is ill--posed
because the function is invariant to at least the $\Orth(d)$ action
(\cref{prop:illposed}), and that the admissible targets of \cref{sec:invariants}
are exactly the invariants of $\Orth(d)\times\R_+$. Matching the diagonal part
requires the alignment maps of \cref{sec:align}, not the relational invariants.

\section{\texorpdfstring{Two objections to calling $\HH$ the ``representation''}{Two objections to calling H the representation}}
\label{app:naming}

The main text fixes our usage---$\HH$ (activations) is the
representation, $\WW$ the reader. Two readers will each expect the word to mean
something else; we address both.

\emph{The representation--theory reader} expects a homomorphism
$\varrho:\Grp\to GL(V)$ realizing an abstract group by linear maps. That is not
what $\HH$ is. A group does appear---$\Grp=\Orth(d)\times\R_+$---but it
\emph{acts on} the feature matrix $\HH$ by $h\mapsto c\,h\QQ$; $\HH$ does not
\emph{represent} $\Grp$. In representation--theory terms the relevant object is
the $\Grp$--action and its orbits, not a representation of $\Grp$ on a vector
space; $\HH$ carries no such homomorphism, it is data. The ``representation
equivalence class'' means the orbit of $\HH$ under this action
(\cref{def:equiv}), with nothing about characters, irreducibles, or intertwiners
intended.

\emph{The computer--science reader} may expect the opposite assignment: that the
\emph{weights} $\WW$---the fixed, shipped, linear--operator part of the model---are
its ``representation,'' while the activations $\HH$ are mere intermediate values.
That intuition is reasonable ($\WW$ is what one publishes, \cref{app:released};
a linear map is closer to the representation--theory sense). We use the opposite
convention deliberately and for one reason: it is the established usage in
representation learning, where ``learned representation,'' ``hidden
representation,'' and the entire representation--similarity literature ($\CKA$ and
its relatives, \cref{sec:invariants}) name the layer activations $\HH$, not the
weights. Our results are stated in that vocabulary so they connect to it.

\section{The claims on a plain MLP, where \cref{ass:absorb} is exact}
\label{app:mlp}

The paper's account is not specific to transformers; it is a property of
representation learning. To make this concrete---and to remove the one
approximation the main text carries (RMSNorm's learned gain does not commute
with a general rotation, \cref{app:status})---we reproduce the four core claims
on a $2$--layer MLP with \emph{no normalization}, where \cref{ass:absorb} holds
\emph{exactly}: a hidden rotation is absorbed by the next layer's weights to
machine precision. The network is
$x\!\to\!\bm{W}_1\!\to\!\mathrm{ReLU}\!\to\! h\!\to\!\bm{W}_2\!\to\!\text{logits}$,
trained on a two--moons task ($\sim\!0.998$ train accuracy); $h\in\R^{N\times16}$
is the representation. Forward and backward passes are hand--written in NumPy
(no autograd), so every step is inspectable. \Cref{tab:mlp} collects the
results.

\begin{table}[h]
\centering
\caption{The four claims on a normalization--free MLP. \cref{ass:absorb} is
exact here, so the invariances hold to machine precision.}
\label{tab:mlp}
\small
\begin{tabular}{p{0.30\textwidth} p{0.62\textwidth}}
\toprule
Claim & Result on the MLP\\
\midrule
(1) \cref{prop:illposed}: absolute matching ill--posed &
A function--preserving rotation $h\!\mapsto\!hQ$ sends
$\norm{h-hQ}_F^2/N$ from $0$ to $28.9$ across $t\in[0,1]$, while
$1-\CKA$ and $L_{\mathrm{rel}}$ stay at $0$ (to $<\!10^{-9}$) for every $t$.\\
\addlinespace
(2) \cref{ass:absorb}/\cref{prop:outinv}: joint invariance &
The cancellation~(\ref{eq:absorb}a) is exact here ($\tilde\psi=\mathrm{id}$,
no normalization): $\norm{(h\QQ)(\QQ^\top\bm{W}_2)-h\bm{W}_2}\approx10^{-14}$.
Hence its consequence~(\ref{eq:absorb}b)---the joint reparametrization
($h\!\mapsto\!h\QQ$, $\bm{W}_2\!\mapsto\!\QQ^\top\bm{W}_2$)---preserves the
function: $\max\lvert f^{\QQ}-f\rvert\approx10^{-14}$. Rotating $h$ \emph{without}
compensating $\bm{W}_2$---call this network $f_{\text{rot}}$, i.e.\ $h\mapsto
h\QQ$ with $\bm{W}_2$ left unchanged---changes it:
$\max\lvert f_{\text{rot}}-f\rvert=25.5$. Invariance is joint, not of $h$ alone (\cref{rem:joint}).\\
\addlinespace
(3) Representation $\neq$ function (the ablation) &
A student hidden layer trained to match the teacher representation by a
basis--invariant objective reaches $\CKA=0.79$ but $\mathrm{KL}(T\Vert S)=6.18$,
top--1 agreement $0.50$. Trained instead to match the output function (logits),
it reaches $\mathrm{KL}=7\times10^{-4}$, top--1 $\approx\!1.00$. Aligning geometry and
restoring function are near--orthogonal, as in \cref{tab:ablation}.\\
\addlinespace
(4) Procrustes alignment &
For $\HH_r=\HH\QQ$, solving
$\QQ_P=\arg\min_{\QQ'}\norm{\HH-\HH_r \QQ'}_F$ reduces the relative
residual from $1.41$ to $\approx10^{-14}$ and recovers $\QQ_P=\QQ^\top$ to
$\approx10^{-12}$: alignment finds the rotation, after which coordinate matching
is well--posed.\\
\bottomrule
\end{tabular}
\end{table}

The contrast with the main experiments is the point. In a transformer,
\cref{ass:absorb} is exact only on the $\bm{\gamma}$--preserving subgroup and
approximate otherwise, yet the same phenomena appear (\cref{tab:rotation,tab:ablation}).
In the MLP there is no $\bm{\gamma}$, the assumption is exact, and the
invariances hold to machine precision. That the qualitative story is identical
in both---feature matching penalizes a function--preserving change the
invariants ignore (1), the output is invariant only under the joint action
(2), aligning the representation does not restore the function (3), and an
explicit alignment makes coordinate matching well--posed (4)---is evidence that
the account is structural, not an artifact of any particular architecture. The
script (NumPy only, no autograd) runs in seconds and reproduces these numbers
to the reported precision across independent machines.

\section{Orbit, fiber, and quotient as runnable code}
\label{sec:code-appendix}

The three geometric objects the paper is built on---\emph{orbit}, \emph{fiber},
and \emph{quotient}---are, for a reader who thinks in programs, three short
pieces of code. This appendix gives a self--contained NumPy script (seed--fixed,
no autograd) that constructs each one and prints what it does; \cref{rem:csorbit}
gives the same correspondences in prose.

We build them in the order a computer scientist finds easiest---starting from
the two notions already familiar and ending at the two that are not. \emph{(i)~A
feature is an IR.} The vector \texttt{h} below is an internal representation, the
network's analogue of a compiler's SSA form: a scratch encoding the computation
uses, not the answer it returns. \emph{(ii)~The invariant is the output.} Rewrite
that IR---rotate \texttt{h}, and update the reader to match---and the returned
value is unchanged; the output function is what stays fixed, exactly as an
executable is unchanged by renaming a program's local variables. These two steps
need no group theory. \emph{(iii)~The set of IR--rewrites sharing one output is an
orbit,} and \emph{(iv)~collapsing each orbit to a single object is a quotient}---and
here the payoff is a distinction code makes unavoidable: there are \emph{two}
quotients, a coarser one that sees only the IR's geometry (what $\CKA$ measures)
and a finer one that sees the output function, and capability lives only on the
finer one. Everything below runs from one setup: a single hidden representation
\texttt{h} and the reader \texttt{W2} that consumes it, exactly as in the
introduction's three--line puzzle.
\begin{lstlisting}
import numpy as np, hashlib
rng = np.random.default_rng(0)
N, d, V = 6, 8, 4
x  = rng.standard_normal((N, d)); W1 = rng.standard_normal((d, d))
W2 = rng.standard_normal((d, V))
h  = x @ W1                       # a hidden representation  (N x d)
y  = h @ W2                       # the output it produces   (N x V)

def rotation(n):                  # a random orthogonal change of basis
    Q, _ = np.linalg.qr(rng.standard_normal((n, n))); return Q
def clone(h, W2, Q, c):           # the joint rewrite: absorb (Q, c)
    return c * (h @ Q), (1.0 / c) * (Q.T @ W2)
def out_hash(h, W2):              # a 'behavior' = hash of the output function
    return hashlib.sha256(np.round(h @ W2, 6).tobytes()).hexdigest()[:12]
def gram_id(h):                   # canonical id of the REPRESENTATION quotient
    G = h @ h.T                   # h h^T is invariant to h -> c h Q
    return hashlib.sha256(np.round(G/np.linalg.norm(G),6).tobytes()).hexdigest()[:12]
\end{lstlisting}

\paragraph{Feature $=$ IR; invariant $=$ output.}
The starting point, before any orbit or quotient: the feature \texttt{h} is the
network's \emph{intermediate representation}, and rewriting it---rotating and
rescaling, with the reader updated in compensation---is a semantics--preserving
transformation. The IR changes substantially ($\lVert h_{\text{ir}}-h\rVert$
large) while the output is bit--for--bit the same (identical hash), the way
renaming an SSA temporary leaves the emitted executable unchanged. The output is
the invariant; the feature is not.
\begin{lstlisting}
r = np.random.default_rng(12345)           # isolated: doesn't perturb below
Q0, _ = np.linalg.qr(r.standard_normal((d, d)));  c0 = 1.7
h_ir, W2_ir = clone(h, W2, Q0, c0)         # re-coordinatize the IR
print(f"IR changed?  ||h_ir - h|| = {np.linalg.norm(h_ir - h):.3e}")
print(f"output same? {np.allclose(h_ir @ W2_ir, y)}  "
      f"(hash {out_hash(h, W2)} -> {out_hash(h_ir, W2_ir)})")
\end{lstlisting}
\begin{lstlisting}[language=,basicstyle=\ttfamily\footnotesize,commentstyle=\ttfamily]
IR changed?  ||h_ir - h|| = 3.375e+01
output same? True  (hash 4f39011d0b0d -> 4f39011d0b0d)
\end{lstlisting}

\paragraph{Orbit $=$ all implementations of one program.}
An orbit is the set of networks obtained from $(h,\WW_2)$ by the joint rewrite
$(h,\WW_2)\mapsto(c\,h\QQ,\ c^{-1}\QQ^\top\WW_2)$. Each stores wildly different
numbers in its hidden layer (large $\lVert h_{\text{var}}-h\rVert$) yet computes
the identical output (identical hash), which is the orbit's defining property.
\begin{lstlisting}
print(f"{'variant':<20}{'||h_var - h||':>15}{'output hash':>15}{'==y?':>7}")
print(f"{'original':<20}{0.0:>15.3e}{out_hash(h,W2):>15}{'yes':>7}")
for k in range(1, 4):
    Q, c = rotation(d), float(rng.uniform(0.5, 2.0))
    hk, W2k = clone(h, W2, Q, c)
    print(f"{'clone %d'%k:<20}{np.linalg.norm(hk-h):>15.3e}"
          f"{out_hash(hk,W2k):>15}{('yes' if np.allclose(hk@W2k,y) else 'NO'):>7}")
\end{lstlisting}
\begin{lstlisting}[language=,basicstyle=\ttfamily\footnotesize,commentstyle=\ttfamily]
variant               ||h_var - h||    output hash   ==y?
original                  0.000e+00   4f39011d0b0d    yes
clone 1                   2.618e+01   4f39011d0b0d    yes
clone 2                   3.017e+01   4f39011d0b0d    yes
clone 3                   1.963e+01   4f39011d0b0d    yes
\end{lstlisting}

\paragraph{Fiber $=$ everything that maps to one output.}
Fixing a behavior (here, the output--hash of $y$), its \emph{fiber} is the
preimage under the output map---every $(h,\WW_2)$ producing that behavior. All
$200$ joint rewrites land in it, because the rewrite is exactly the
function--preserving move; a bare rotation that does \emph{not} compensate the
reader leaves the fiber. This is the code form of ``same executable / same hash /
same API'': the fiber is the set of sources that build to one artifact.
\begin{lstlisting}
inside = sum(out_hash(*clone(h, W2, rotation(d), float(rng.uniform(0.5,2.0)))) 
             == out_hash(h, W2)  for _ in range(200))
print(f"joint rewrites landing in fiber of y : {inside}/200")
y_rot = (h @ rotation(d)) @ W2     # rotate h but DON'T fix the reader
print("rotation without compensating reader :",
      "in" if np.allclose(y_rot, y) else "OUTSIDE", "the fiber")
\end{lstlisting}
\begin{lstlisting}[language=,basicstyle=\ttfamily\footnotesize,commentstyle=\ttfamily]
joint rewrites landing in fiber of y : 200/200
rotation without compensating reader : OUTSIDE the fiber
\end{lstlisting}

\paragraph{Quotient $=$ collapse each orbit to one point---but which quotient?}
This is the paper's central distinction (\cref{prop:quotient},
\cref{fig:orbit}) made executable. Take \emph{two} programs, $A$ and $B$, that
share the same hidden representation $h$ but have \emph{different} readers
($\WW_2$ vs.\ a fresh $\WW_2^B$), and make $30$ clones of each. Collapsing by the
representation quotient---keyed on the Gram invariant $h\,h^\top$, which is what
$\CKA$ sees---merges $A$ and $B$ into \emph{one} class: geometry alone cannot
tell them apart. Collapsing by the joint quotient---keyed on the output
function---keeps them as \emph{two} pure classes. Capability lives on the finer,
joint quotient; the representation quotient is strictly coarser.
\begin{lstlisting}
W2b = rng.standard_normal((d, V))                 # program B: different reader
rows = []
for name, ww in (("A", W2), ("B", W2b)):
    for _ in range(30):
        Q, c = rotation(d), float(rng.uniform(0.5, 2.0))
        rows.append((name, *clone(h, ww, Q, c)))
joint, repq = {}, {}
for name, hk, wk in rows:
    joint.setdefault(out_hash(hk, wk), set()).add(name)
    repq.setdefault(gram_id(hk),      set()).add(name)
print(f"representation quotient (gram) -> {len(repq)} class ; "
      f"A and B collapsed together: {any(len(v) > 1 for v in repq.values())}")
print(f"joint quotient (output func)   -> {len(joint)} classes; "
      f"every class pure (one program): {all(len(v) == 1 for v in joint.values())}")
\end{lstlisting}
\begin{lstlisting}[language=,basicstyle=\ttfamily\footnotesize,commentstyle=\ttfamily]
representation quotient (gram) -> 1 class ; A and B collapsed together: True
joint quotient (output func)   -> 2 classes; every class pure (one program): True
\end{lstlisting}
The last block is \cref{fig:orbit} in code: $A$ and $B$ sit in the same
representation--quotient point (the same ``meeting room,'' \cref{rem:joint}) yet
different joint--quotient points (different ``seats''), and it is the seat---the
output function---that capability reads. A distillation objective defined on the
Gram invariant cannot distinguish $A$ from $B$; one defined on the output
function can. That is the paper's thesis, executable in forty lines. The full
script---printing every intermediate result---is distributed with the paper's
code release.

\end{document}